\documentclass{article}

\usepackage{titletoc}
\usepackage[page, header, toc, page]{appendix} 

\usepackage{wrapfig}
\usepackage{tikz}
\usepackage{pgfplots}
\pgfplotsset{compat=1.18}
\usepackage{subcaption}

\usepackage[final]{neurips_2025}

\PassOptionsToPackage{compress, sort}{natbib}

\usepackage{microtype}
\usepackage{graphicx}
\usepackage{float,caption,subcaption}
\usepackage{booktabs} 
\usepackage{algorithm}
\usepackage{algorithmic}
\usepackage{enumitem}
\usepackage{nicefrac} 

\usepackage{tcolorbox}

\PassOptionsToPackage{x11names}{xcolor}
\usepackage[x11names]{xcolor} 
\definecolor{darkblue}{rgb}{0.0,0.0,0.65}
\definecolor{darkred}{rgb}{0.65,0.0,0.0}
\definecolor{darkgreen}{rgb}{0.0,0.5,0.0}
\definecolor{tab:blue}{RGB}{31,119,180}  
\definecolor{tab:red}{RGB}{214,39,40}  
\definecolor{tab:green}{RGB}{44,160,44}  
\definecolor{tab:orange}{RGB}{255,127,14}  
\definecolor{background}{HTML}{EDE7DE}

\usepackage{mdframed}
\mdfsetup{
linecolor=black!40,
linewidth=0.5pt,
innerleftmargin = 5pt,
innertopmargin = 5pt,
innerrightmargin= 5pt,
innerbottommargin = 5pt,
}

\usepackage[pagebackref]{hyperref}
\hypersetup{
	colorlinks = true,
	citecolor  = darkblue,
	linkcolor  = darkblue,
	filecolor  = darkblue,
	urlcolor   = darkblue,
}

\usepackage{amsmath}
\usepackage{amssymb}
\usepackage{mathtools}
\usepackage{amsthm}

\usepackage{multirow}
\usepackage{bbm}
\usepackage{enumitem}
\usepackage{float}
\usepackage{thmtools, thm-restate}

\usepackage[capitalize,noabbrev]{cleveref}

\theoremstyle{plain}
\newtheorem{theorem}{Theorem}[section]

\newtheorem{lemma}[theorem]{Lemma}

\theoremstyle{definition}

\theoremstyle{remark}

\usepackage{amsmath,amsfonts,amssymb,bm,bbm,pifont}

\def\inner#1{\left\langle#1\right\rangle}

\def\1{\bm{1}}

\DeclareMathAlphabet{\mathsfit}{\encodingdefault}{\sfdefault}{m}{sl}
\SetMathAlphabet{\mathsfit}{bold}{\encodingdefault}{\sfdefault}{bx}{n}

\newcommand{\E}{\mathbb{E}}

\usepackage[utf8]{inputenc} 
\usepackage[T1]{fontenc}    

\setlist{nolistsep}
\setlist{leftmargin=*} 

\def\bz{{\mathbf z}}

\def\bx{{\mathbf x}}

\def\bH{{\mathbf H}}
\def\bP{{\mathbf P}}
\def\bI{{\mathbf I}}

\def\bw{{\mathbf w}}

\def\by{{\mathbf y}}
\def\bg{{\mathbf g}}

\definecolor{cornflowerblue}{rgb}{0.39, 0.58, 0.93}

\crefname{observation}{Observation}{Observations}
\newcounter{observation}
\newcommand{\newobservation}[2][]{%
  \refstepcounter{observation}%
  \ifx&#1&
  \else
    \label{#1}
  \fi
  \begin{tcolorbox}[colback=cornflowerblue!20!white, colframe=white,
    left=1pt,
    right=1pt,
    top=1pt,
    bottom=1pt]
  {\textbf{\emph{Observation \theobservation:} }#2}
  \end{tcolorbox}
}

\title{Through the River: Understanding the Benefit of Schedule-Free Methods for Language Model Training}

\author{%
  Minhak Song\thanks{Equal contribution.} \\
  KAIST\\
  \texttt{minhaksong@kaist.ac.kr} \\
  \And
  Beomhan Baek\footnotemark[1]~~\thanks{Work done as an undergraduate intern at KAIST.} \\
  SNU \& KAIST InnoCORE LLM \\
  \texttt{bhbaek2001@snu.ac.kr} \\
  \AND
  Kwangjun Ahn \\
  Microsoft Research \\
  \texttt{kwangjunahn@microsoft.com} \\
  \And
  Chulhee Yun \\
  KAIST \\
  \texttt{chulhee.yun@kaist.ac.kr} \\
}

\begin{document}

\maketitle

\begin{abstract}
As both model and dataset sizes continue to scale rapidly, conventional pretraining strategies with fixed compute budgets---such as cosine learning rate schedules---are increasingly inadequate for large-scale training. Recent alternatives, including warmup-stable-decay (WSD) schedules and weight averaging, offer greater flexibility. However, WSD relies on explicit decay phases to track progress, while weight averaging addresses this limitation at the cost of additional memory. In search of a more principled and scalable alternative, we revisit the Schedule-Free (SF) method~\citep{defazio2024road}, which has shown strong empirical performance across diverse settings. We show that SF-AdamW effectively navigates the ``river'' structure of the loss landscape without decay phases or auxiliary averaging, making it particularly suitable for continuously scaling training workloads. To understand this behavior, we conduct a theoretical and empirical analysis of SF dynamics, revealing that it implicitly performs weight averaging without memory overhead. Guided by this analysis, we propose a refined variant of SF that improves robustness to momentum and performs better under large batch sizes, addressing key limitations of the original method. Together, these results establish SF as a practical, scalable, and theoretically grounded approach for language model training.
\end{abstract}

\section{Introduction} 

As both model and dataset sizes continue to scale rapidly, conventional pretraining strategies with fixed training budgets---such as cosine learning rate schedules~\citep{loshchilov2017sgdr}---are becoming increasingly inadequate. These static approaches are ill-suited to the demands of large-scale, evolving datasets and open-ended training regimes. For example, DeepSeek-V3~\citep[\S4.2]{liu2024deepseek} employs a sophisticated multi-phase training procedure that falls outside the scope of traditional cosine scheduling.

To support prolonged and flexible training, practitioners have adopted more adaptive scheduling strategies. One widely used approach is the \emph{warmup-stable-decay} (WSD) schedule~\citep{hu2024minicpm}, which avoids committing to a fixed compute budget by maintaining a main ``branch'' with a constant learning rate (LR) and periodically branching into decaying LR trajectories to produce intermediate checkpoints---enabling flexible and continued training.
Despite its advantages, WSD has notable limitations. A key challenge lies in evaluating the quality of the current model without explicitly entering the decay phase. This lack of visibility complicates decisions around checkpointing and training continuation, leading to uncertainty in training management (see \cref{sec:wsd}).

One common workaround is to maintain a weight averaging, which improves generalization and provides more stable performance estimation. However, this comes at the cost of additional memory overhead and implementation complexity, especially in distributed training setups (see \cref{sec:wa}).  

These challenges motivate a key question:
\begin{center}
\emph{Is there an alternative that offers better flexibility, training visibility, and minimal resource overhead?}
\end{center}

\paragraph{Our main contributions.}
In this work, we explore this question and identify the Schedule-Free (SF) method~\citep{defazio2024road} as a principled and scalable approach for language model pretraining. Our contributions are summarized as follows:

\begin{itemize}
    \item We revisit the river-valley loss landscape and analyze two widely used strategies---WSD and weight averaging---through this lens, highlighting their respective strengths and limitations for scalable pretraining (\cref{sec:continuous}).
    \item We then focus on the SF method and show that it effectively follows the ``river'' structure of the loss landscape without requiring a decay phase or auxiliary averaging. This makes it particularly well-suited for continuously scaling training workloads (\cref{sec:sf_critiques}).
    \item We analyze its training dynamics both theoretically and empirically. Our findings reveal that SF implicitly performs a form of weight averaging, without requiring additional memory. We also show that it operates at the Edge of Stability~\citep{cohen2021gradient} and derive its associated central flow~\citep{cohen2025understanding}, providing a deeper understanding of its behavior (\cref{sec:sf_understand}).
    \item Based on these insights, we propose a refined version of the SF method that improves robustness to momentum parameters and scale better with large batch sizes---addressing key limitations of the original method (\cref{sec:refined}).
\end{itemize}

\section{Candidate Strategies for Scalable Pretraining}
\label{sec:continuous}

\subsection{Backgrounds: Loss Landscape of Neural Networks}

\begin{wrapfigure}{r}{0.4\textwidth}
    \vspace{-12pt}  
    \centering
     \includegraphics[width=\linewidth]{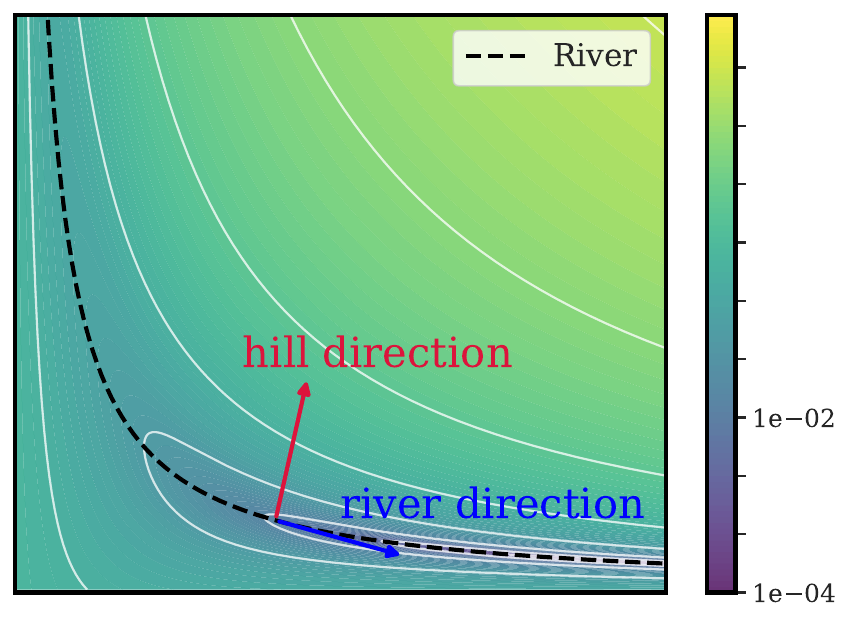}
    \vspace{-10pt}
    \caption{\textbf{River-valley structure in a toy loss landscape.} Contour plot of the objective defined in \cref{sec:toy}, illustrating the flat river direction and steep hill direction characteristic of the river-valley geometry.}
    \label{fig:river-valley}
    \vspace{-25pt}
\end{wrapfigure}
 
Despite the remarkable success of deep learning across numerous domains, classical optimization theory falls short of explaining the underlying dynamics of neural network training. In particular, the structure of the loss landscape has attracted growing attention as researchers seek to better understand why deep learning works and how to design more effective optimization algorithms.

\paragraph{River-Valley Landscape.}
Recent studies—motivated by diverse goals—have converged on a common hypothesis regarding the geometry of neural network loss landscape. This hypothesis, which we refer to as the \textit{river-valley loss landscape} \citep{wen2024understanding}, is closely related to concepts such as the \textit{ill-conditioned valley} \citep{song2025does}, \textit{basin in the loss landscape} \citep{hagele2024scaling}, and \textit{ravine in the loss landscape} \citep{davis2024gradientdescentadaptivestepsize}.

As its name suggests, a \emph{river-valley} loss landscape resembles a winding ravine: steep ``hill'' walls flank a relatively flat ``river'' floor that snakes through parameter space. \citet{wen2024understanding} formalize this picture to interpret the warmup-stable-decay (WSD) LR schedule, arguing that large, noisy updates during the \emph{stable} phase let SGD travel quickly \emph{downstream} along the river, while the fast‑decay phase pulls the iterate back to the bottom of valley. 
A concrete illustration is provided in \cref{fig:river-valley} (see also \cref{sec:toy}).
Complementary evidence from \citet{song2025does} show that ``effective'' optimizer updates happen along the river: projecting out the high‑curvature hill directions does not harm progress of learning, indicating that motion in the hill directions is often dispensable once the iterate is near the river.

Motivated by these works, we explicitly decompose the loss into two orthogonal parts:
\begin{itemize}[noitemsep]
\item the \textbf{river component}, which measures progress \emph{along} the low‑curvature valley floor, and
\item the \textbf{hill component}, which penalizes deviations \emph{away} from that floor.
\end{itemize} 

In the remainder of this section, we revisit existing strategies for scalable pretraining through the lens of the river-valley perspective.

\subsection{Warmup-Stable-Decay Schedule}\label{sec:wsd}

To address the limitations of cosine schedules---particularly their reliance on a pre-specified training budget---the warmup-stable-decay (WSD) schedule has been proposed as a more flexible alternative~\citep{zhai2022scaling, hu2024minicpm}.
WSD divides into three phases: warmup, stable, and decay, with the LR controlled separately in each. Unlike traditional schedules, WSD avoids committing to a fixed training horizon by maintaining a main branch with a constant LR and periodically branching off with decaying LRs to produce intermediate checkpoints. This structure enables flexible evaluation and checkpointing without the need to predefine the total number of training steps, making WSD a widely adopted strategy for scalable pretraining.

\paragraph{Understanding WSD.}
The WSD LR schedule has been widely adopted in large language model (LLM) pretraining due to its strong empirical performance. Motivated by this success, recent studies have sought to understand the mechanisms underlying its effectiveness. \citet{hagele2024scaling} systematically explore the impact of WSD hyperparameters and propose optimal choices for both the decay schedule and timing. In parallel, \citet{luo2025a} introduce a multi-power-law framework that predicts final pretraining loss as a function of the LR schedule. Interestingly, their learned optimal schedules closely resemble the WSD pattern, providing further evidence for its effectiveness.

\citet{wen2024understanding} provide a geometric interpretation of WSD through the lens of the river-valley loss landscape. Their key insights are:

\begin{enumerate}
\item During the stable phase, the high LR amplifies stochastic gradient noise, inducing oscillations along the high-curvature hill directions. Nevertheless, it enables rapid progress along the low-curvature river direction, which drives long-term improvement.
\item The decay phase plays a crucial role near convergence: it reduces oscillations in the hill directions and steers the iterate toward the valley floor, resulting in a sharp drop in loss that is not achievable during the stable phase alone.
\end{enumerate}

\paragraph{Limitations.}
A key limitation of WSD is its reliance on manually initiating the decay phase. While the stable phase often yields a relatively flat loss curve, a sharp drop typically occurs only once decay begins, which makes it difficult to assess model quality or forecast final performance in advance. This raises a natural question: can we design optimizers that closely track optimal loss---by reaching the valley floor and following the river---\emph{without} relying on explicit learning rate decay?

\subsection{Weight Averaging} \label{sec:wa}

Since its introduction in stochastic approximation~\citep{ruppert1988efficient, polyak1992acceleration}, parameter averaging has been widely explored for improving optimization stability and generalization in deep learning. By reducing gradient noise and smoothing the optimization trajectory, averaging schemes can often eliminate the need for explicit LR decay, making them an appealing candidate for scalable pretraining. Among them, two widely studied approaches stand out:

\begin{enumerate}
  \item \textbf{Stochastic Weight Averaging (SWA). } SWA~\citep{izmailov2019averagingweightsleadswider} enhances generalization by periodically averaging model weights. While the original method uses a cyclic LR schedule and averages every $c$ steps, many subsequent works simplify it by setting $c=1$,  performing a standard running averaging. \citet{hagele2024scaling} further refine SWA by applying fixed-length window averaging under constant LR, and demonstrated improved performance. 

  \item \textbf{Exponential Weight Averaging (EWA). } EWA maintains an exponential moving average of model weights, continuously smoothing the optimization trajectory. Recently, \citet{zhang2025how} show that combining EWA with a constant LR match the performance of cosine schedulers and WSD, particularly in large-batch settings. EWA has also been proven to be effective \emph{theoretically}  in nonconvex optimization~\citep{ahn2024adam, ahn2024general}.
\end{enumerate}

Interestingly, weight averaging is often regarded as functionally equivalent to LR decay. \citet{sandler2023trainingtrajectoriesminibatchlosses} analyze schemes like SWA and EWA, showing that their dynamics closely resemble those induced by decaying LRs. Motivated by this, several works advocate for weight averaging as a viable alternative to schedulers in scalable pretraining~\citep[e.g.,][\S4.1]{hagele2024scaling}. Concurrent with our work, \citet{li2025model} report that training with a constant LR, when combined with weight averaging, matches the performance of models trained with decaying schedules at any point during training, without the need for LR decay. From the river-valley perspective, weight averaging serves to cancel out oscillations along hill directions, enabling the optimization trajectory to align more closely with the river---\emph{without} relying on explicit LR decay.

\paragraph{Limitations.}
Despite its benefits, weight averaging introduces a memory overhead, as it requires maintaining an additional copy of the model parameters. This becomes a bottleneck in large-scale LLM pretraining. For instance, storing a single 16-bit copy of LLaMA-8B requires over 16 GB of memory. This limits the practicality of weight averaging in memory-constrained or large-scale training environments.

\subsection{Schedule-Free Methods}\label{sec:sf}

The Schedule-Free (SF) method~\citep{defazio2024road} provides a general framework that interpolates between two classical techniques: Polyak-Ruppert averaging, which returns the average of past iterates, and primal averaging, where gradients are evaluated at the averaged point. The abstract formulation of the SF method is given by:
\begin{align} 
\begin{split}
\bx_{t} &= (1-c_{t})\, \bx_{t-1} + c_{t}\, \bz_{t},\\[1mm]
\by_t &= (1-\beta)\, \bz_t + \beta\, \bx_t,\\[1mm]
\bz_{t+1} &= \bz_t - \gamma \Delta_t,
\end{split}
\tag{\texttt{SF}}\label{eq:schefree}
\end{align}
where $\gamma$ is the LR, $\beta$ is a momentum-like coefficient, $c_t = 1/t$, and the initialization satisfies $\bz_1 = \bx_1$. The update direction $\Delta_t$ is generic, making SF a flexible framework that can be combined with any baseline optimizer. For example, in Schedule-Free SGD, $\Delta_t$ corresponds to a stochastic gradient evaluated at the $\by_t$ iterate.

In this work, we focus on Schedule-Free AdamW (\texttt{SF-AdamW}), where $\Delta_t$ is computed using the RMSprop update along with a weight decay term. The full pseudocode is provided in \cref{alg:sf_adamw}. Here, $\beta_1$ denotes the coefficient $\beta$ in \eqref{eq:schefree}, and $\beta_2$ is the momentum parameter used in the second-moment of RMSprop.

\texttt{SF-AdamW} has demonstrated state-of-the-art performance across a range of deep learning tasks, including winning the Self-Tuning track in the 2024 AlgoPerf Challenge~\citep{dahl2023benchmarking,kasimbeg2025accelerating}. Importantly, it achieves this without requiring additional memory overhead compared to \texttt{AdamW}. However, its practical deployment reveals two key limitations: sensitivity to momentum hyperparameters~\citep{hagele2024scaling} and degraded performance under large batch sizes~\citep{zhang2025how,morwani2025connections}. We revisit both limitations in later sections and propose a refined variant of SF that addresses them.

\section{Schedule-Free Optimizer as a Scalable Pretraining Method} \label{sec:sf_critiques}

In \cref{sec:continuous}, we discussed the limitations of WSD and weight averaging as strategies for scalable pretraining. While WSD relies on a decay phase to achieve optimal performance, weight averaging avoids decay but incurs additional memory overhead. In this section, we empirically investigate the SF method as an alternative. We find it to be a strong candidate, as it requires neither decay nor extra memory.

\paragraph{Experimental Setup.}
We use a 124M parameter LLaMA~\citep{touvron2023llama,touvron2023llama2} style decoder-only transformer, trained with \texttt{SF-AdamW} using a warmup phase followed by a constant LR. 
The batch size is 0.5M tokens, and training is conducted on a 6B-token subset of SlimPajama~\citep{cerebras2023slimpajama}, with the compute budget determined by the Chinchilla scaling rule~\citep{hoffmann2022training}. 
We report validation loss in terms of perplexity. Additional results using a 124M parameter GPT-2~\citep{radford2019language} style decoder-only transformer trained on the OpenWebText2 dataset~\citep{pile} are provided in \cref{appendix:additional_results}. Full training details are available in \cref{appendix:experiment}.

\paragraph{Vanishing Benefit of Learning Rate Decay and Weight Averaging.}
As noted in \cref{sec:continuous}, standard \texttt{AdamW} with a constant LR typically yields suboptimal performance, requiring a decay phase or weight averaging to reach better solutions. To test whether \texttt{SF-AdamW} exhibits similar behavior, we first perform a grid search to identify the best hyperparameters for both \texttt{AdamW} and \texttt{SF-AdamW} under constant LR (after warmup). In our setting, the optimal hyperparameters are $(\beta_1, \beta_2)=(0.9, 0.95)$ with LR 1e-3 for \texttt{AdamW}, and $(\beta_1, \beta_2)=(0.95, 0.99)$ with LR 2e-3 for \texttt{SF-AdamW}. Using these configurations, we train each model and periodically save checkpoints. At each checkpoint, we run a LR decay phase and evaluate the resulting loss. We also track the EWA of the $\bx_t$ iterates throughout training. Results are shown in \cref{fig:Obs1}.

Surprisingly, unlike \texttt{AdamW}, \emph{neither the decay phase nor EWA provides additional benefit}: \texttt{SF-AdamW} with a constant LR consistently reaches near-optimal solutions on its own.

\paragraph{Schedule-Free Optimizer Tracks the River.}
We next examine how closely the \texttt{SF-AdamW} trajectory follows the river in the loss landscape, building on the observation by \citet{wen2024understanding} that \texttt{AdamW} with a decaying LR converges toward the river. Specifically, we run a short decay phase of \texttt{AdamW} (linear decay from 1e-4 to 0) at each checkpoint from the \texttt{SF-AdamW} run described in the previous experiment. This decay phase---starting from a small LR---is designed to make minimal progress along the river direction while substantially reducing the hill component, pulling the iterate toward the valley floor. As a baseline, we apply the same procedure to \texttt{AdamW}. Results are shown in \cref{fig:Obs1}. 

We observe that applying a decay phase of \texttt{AdamW} to the \texttt{SF-AdamW} trajectory results in minimal additional loss reduction, in contrast to the \texttt{AdamW} trajectory where the decay phase leads to a sharp drop in loss. This suggests that \texttt{SF-AdamW} already closely tracks the river throughout training, without requiring LR decay or weight averaging.

To further support this interpretation, we measure the loss along linear interpolations between the 2B and 2.5B token checkpoints under three training regimes: (1) \texttt{AdamW} with a constant learning rate (LR), (2) \texttt{AdamW} with a linear LR decay to zero, and (3) \texttt{SF-AdamW} with a constant LR. As shown in \autoref{fig:river_valley_loss_profile}, the resulting loss curves display qualitatively distinct behaviors: (1) shows a convex, valley-shaped profile, indicating oscillation across the valley; (2) shows a sharp, monotonic decline, consistent with a transition from the valley wall to the floor; and (3) shows a flat, slow decline, suggesting that the trajectory is already closely aligned with the valley floor. Notably, cases (1) and (2) replicate the loss profiles reported in Figure~7 of \citet{wen2024understanding}, while (3) offers further evidence that \texttt{SF-AdamW} remains near the river throughout training.

\begin{figure}[t]
    \begin{subfigure}{0.33\linewidth}
        \centering
        \includegraphics[width=\linewidth]{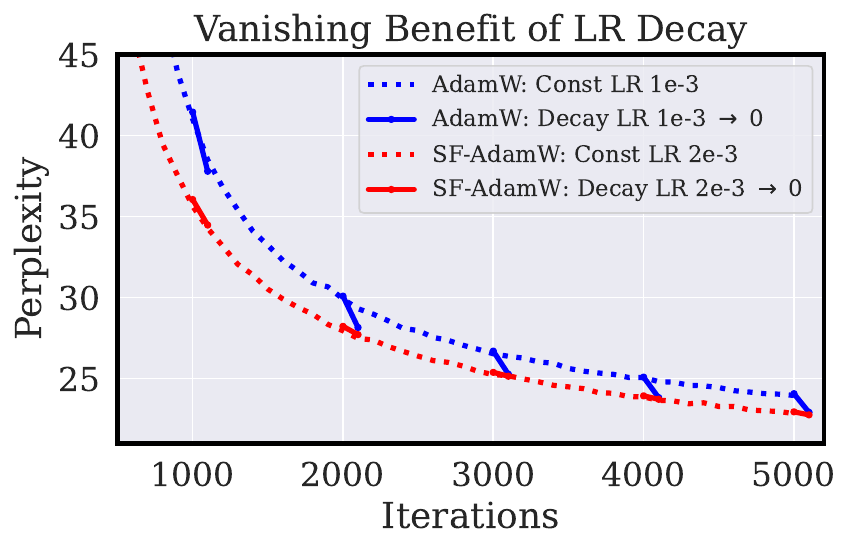}
        \label{fig:LR_decay_slimpajama}
    \end{subfigure}
    \begin{subfigure}{0.33\linewidth}
        \centering
        \includegraphics[width=\linewidth]{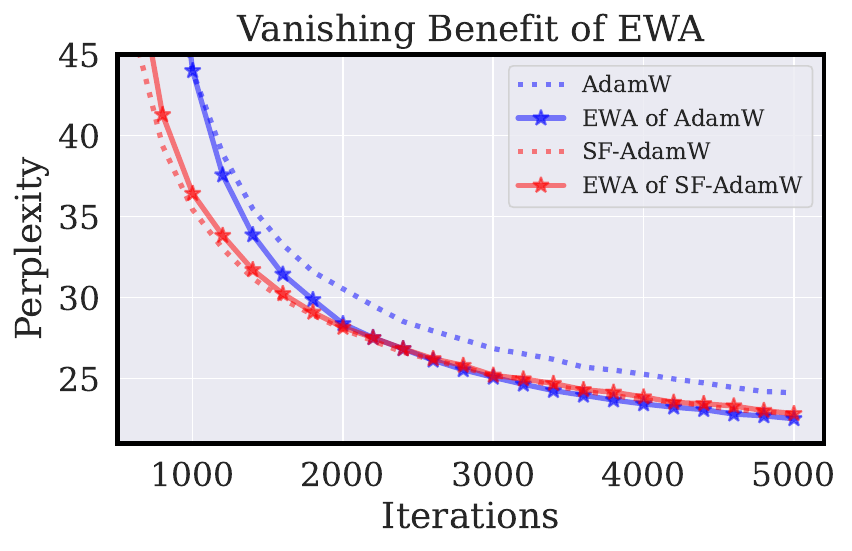}
        \label{fig:EWA_slimpajama}
    \end{subfigure}
    \begin{subfigure}{0.33\linewidth}
        \centering
        \includegraphics[width=\linewidth]{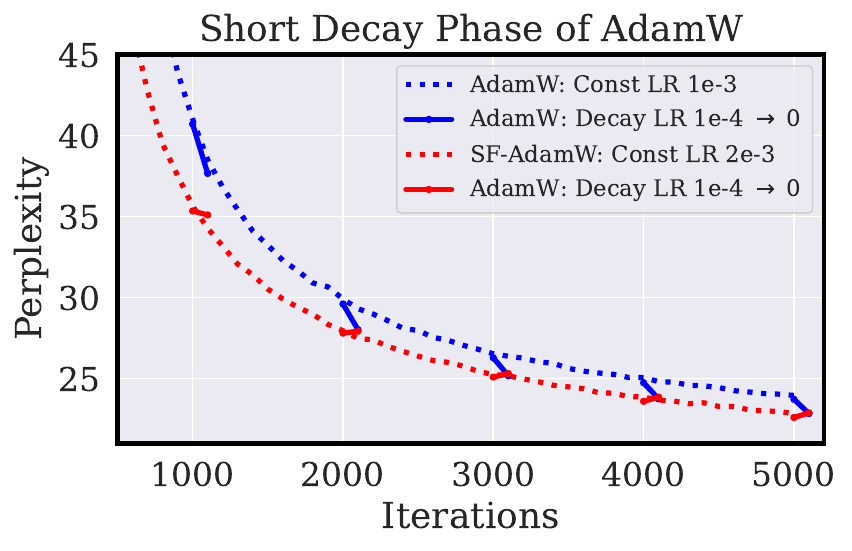}
        \label{fig:Decay_AdamW_slimpajama}
    \end{subfigure}
    \vspace{-20pt}
    \caption{\textbf{SF-AdamW closely follows the river, unlike AdamW.} \textbf{Left, Middle:} While \texttt{AdamW} benefits from linear LR decay and EWA, \texttt{SF-AdamW} shows no improvement from either. \textbf{Right:} A short decay phase of \texttt{AdamW} (with linear LR decay from 1e-4 to 0) leads to a sharp loss drop for \texttt{AdamW}, but has no effect when applied to the \texttt{SF-AdamW} trajectory---suggesting that \texttt{SF-AdamW} already tracks the river throughout training (\cref{obs:sf_river}).}
    \label{fig:Obs1}
    \vspace{-10pt}
\end{figure}

\begin{figure}[!tbh]
    \vspace{-5pt}
    \centering
    \includegraphics[width=\linewidth]{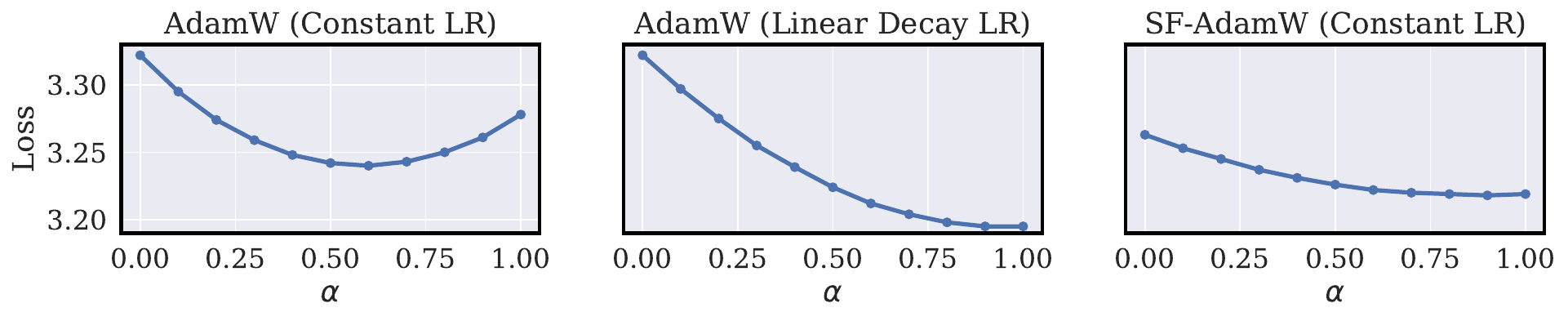}
    \vspace{-15pt}
    \caption{\textbf{Linear interpolation between training checkpoints.}  We evaluate the loss along linear interpolations $\alpha \bw_{t_1} + (1{-}\alpha)\bw_{t_2}$, where $\alpha \in [0, 1]$ and $t_1$, $t_2$ denote the 2B and 2.5B token checkpoints, respectively.  We compare three training regimes: (1) \texttt{AdamW} with constant learning rate (LR), (2) \texttt{AdamW} with a linear LR decay to zero, and (3) \texttt{SF-AdamW} with constant LR.  For all settings, the first 2B tokens are trained using either constant-LR \texttt{AdamW} (for 1 and 2) or constant-LR \texttt{SF-AdamW} (for 3).  The resulting curves exhibit qualitatively distinct behaviors: convex (valley-shaped) for (1), sharp monotonic decay for (2), and flat, slowly declining loss for (3) (\cref{obs:sf_river}).}
    \vspace{-10pt}
\label{fig:river_valley_loss_profile}
\end{figure}

\newobservation[obs:sf_river]{\texttt{SF-AdamW} can follow the river without LR decay or weight averaging.}

\begin{wrapfigure}{r}{0.4\textwidth}
    \vspace{-5pt}  
    \centering
     \includegraphics[width=\linewidth]{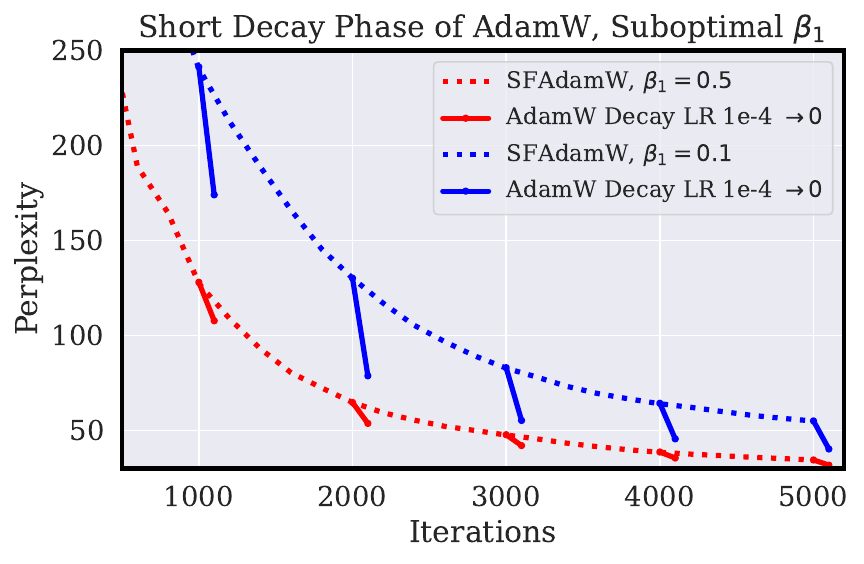}
    \vspace{-15pt}  
    \caption{\textbf{SF-AdamW with suboptimal momentum fails to follow the river.}  A short decay phase of \texttt{AdamW} applied to \texttt{SF-AdamW} checkpoints with $\beta_1 \in \{0.1, 0.5\}$ results in a sharp loss drop, unlike the case with $\beta_1 = 0.95$ (\cref{obs:sensitive-momentum}).}
    \label{fig:Decay_AdamW_betas}
    \vspace{-20pt}
\end{wrapfigure}

\paragraph{Sensitivity to Momentum Parameters.} 

Despite its strong empirical performance, SF method is highly sensitive to the choice of momentum parameters. For example, \citet{hagele2024scaling} report that \texttt{SF-AdamW} with $(\beta_1, \beta_2) = (0.9, 0.95)$ performs significantly worse in their pretraining setup, even exhibiting rising loss toward the end of training, whereas $(0.95, 0.99)$ leads to strong results. This sensitivity contrasts with the theoretical analysis of \citet{defazio2024road}, which shows that the SF method is worst-case optimal for any momentum setting in convex Lipschitz problems.

To further investigate this gap, we repeat our experiment using \texttt{SF-AdamW} with suboptimal momentum $\beta_1 \in \{0.1, 0.5\}$. As before, we periodically save checkpoints and apply a short \texttt{AdamW} decay phase at each. Unlike the optimal case ($\beta_1=0.95$), we observe that the decay phase improves performance, suggesting that suboptimal momentum disrupts the optimizer's ability to follow the river. Results are shown in \cref{fig:Decay_AdamW_betas}.

\newobservation[obs:sensitive-momentum]{\texttt{SF-AdamW} is highly sensitive to momentum; poor choices can prevent it from reaching and following the river.}

These findings lead to the following central question:
\begin{center}
\emph{Why does a well-tuned Schedule-Free method successfully follow the river, and\\ what makes this behavior so sensitive to momentum?}
\end{center}
We address this question in the next section by analyzing the training dynamics of the SF method.

\vspace{-5pt}
\section{Understanding the Training Dynamics of Schedule-Free Optimizer} \label{sec:sf_understand}
\vspace{-5pt}

\subsection{Warmup: A Toy Model}\label{sec:toy}
\vspace{-3pt}

To build intuition, we start by studying the following simple two-dimensional objective:
\[
f(\bw) = \frac{1}{2}(w^{(1)}w^{(2)}-1)^2 + \log(1+\exp(-w^{(1)}))\;,
\]
where $\bw = (w^{(1)}, w^{(2)})\in \mathbb R^2$.
As shown in \cref{fig:river-valley}, this objective exhibits a \emph{river-valley} structure: the curve $w^{(1)}w^{(2)} = 1$ forms a river, along which the loss slowly decreases as $w^{(1)}$ increases. We use this toy model to visualize the training dynamics of \texttt{SF-AdamW} in a controlled setting.

We run \texttt{SF-AdamW} with a constant LR, initialized at $\mathbf x_1 =  (2,2)$, fixing $\beta_2 = 0.99$ and varying $\beta_1 \in \{0.1, 0.5, 0.9\}$. For each run, we track the iterates $(\bx_t, \by_t, \bz_t)$. Results are shown in \cref{fig:toy_model}.

We observe that the $\by_t$ iterates---where the gradient is evaluated---oscillate around the river across all momentum settings, with the center of oscillation remaining close to the river. In contrast, the $\bx_t$ iterates fail to track the river for suboptimal values of $\beta_1$ ($0.1$ and $0.5$), whereas $\beta_1 = 0.9$ results in a trajectory that remain closely aligned with the river.

This behavior aligns with \cref{obs:sensitive-momentum}, where we observed that \texttt{SF-AdamW} fails to follow the river under suboptimal momentum configurations in language model training. Notably, even when $\beta_1$ is suboptimal, the $\by_t$ iterates continue to track the river on average, despite exhibiting oscillations (we revisit the nature of this oscillation in \cref{sec:eos}). These observations suggest that the $\by_t$ sequence is more robust to $\beta_1$ and better aligned with the river geometry than the $\bx_t$ iterates, making it a more reliable foundation for analyzing and guiding optimization in SF dynamics. We investigate this further in the context of language model training.

\begin{figure}[t]
    \vspace{-5pt}
    \centering
    \includegraphics[width=\linewidth]{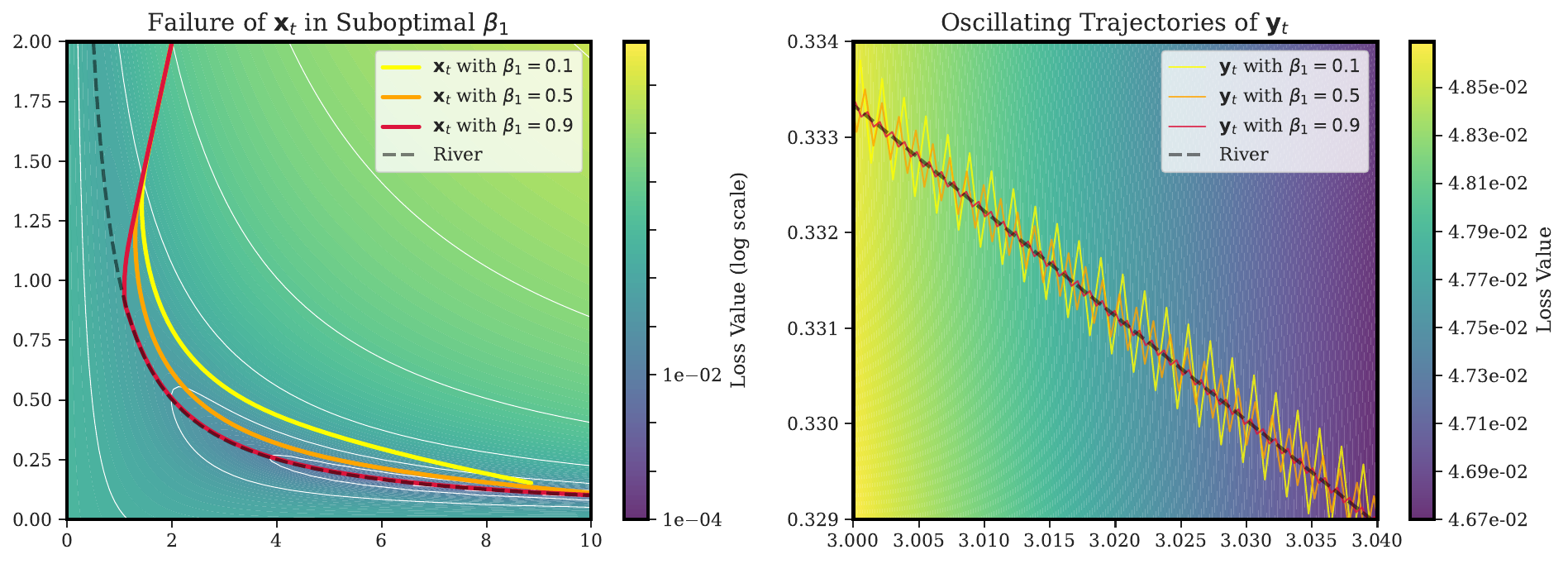}
    \caption{\textbf{SF-AdamW on toy model.} \textbf{Left:} The $\bx_t$ iterates fail to follow the river for $\beta_1 \in \{0.1, 0.5\}$ (\cref{obs:sensitive-momentum}). \textbf{Right:} The $\by_t$ iterates oscillate around the river but track it reliably on average, even for suboptimal values of $\beta_1$ (\cref{obs:y}). As $\beta_1$ increases, the oscillations shrink.}
    \label{fig:toy_model}
    \vspace{-15pt}
\end{figure}

\vspace{-3pt}
\subsection{Language Model Training} \label{sec:language_y}
\vspace{-3pt}

We evaluate the loss at the $\by_t$ iterates, as well as the EWA of $\by_t$, using the same experimental runs from \cref{sec:sf_critiques} with $\beta_1 \in \{0.1, 0.5, 0.95\}$. Results are shown in \cref{fig:EMA_y}. Notably, for suboptimal $\beta_1$, we observe that the loss at $\by_t$ is consistently lower than that at $\bx_t$, showing that $\by_t$ more faithfully follows the river geometry and remains robust to suboptimal momentum settings. This mirrors our findings in the toy model analysis (\cref{sec:toy}), where $\bx_t$ failed to follow the river under suboptimal momentum, while $\by_t$ continued to track it.
Moreover, the EWA of $\by_t$ consistently achieves lower loss than the raw $\by_t$ iterates across all momentum configurations---unlike the $\bx_t$ iterates, where EWA offers no benefit for $\beta_1=0.95$. 
It illustrates that the EWA of $\by_t$ consistently remains closer to the river compared to the vanilla $\by_t$ iterates. This observation parallels the oscillatory behavior of $\by_t$ in the toy model, where the EWA would more closely align with the underlying river-like geometry.

\begin{figure}[!tbh]
    \vspace{-5pt}
    \centering
    \includegraphics[width=\linewidth]{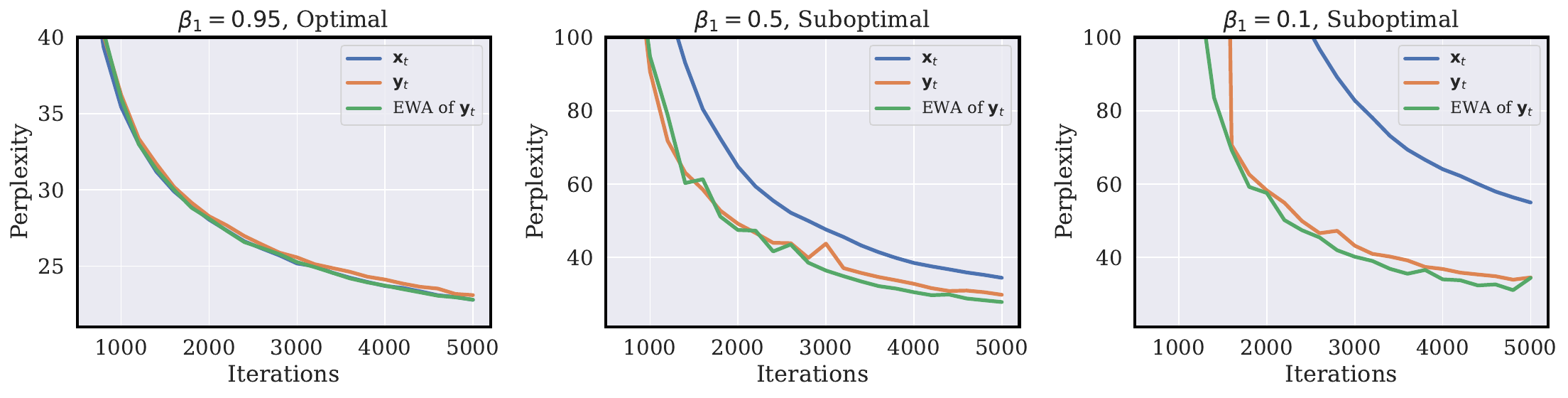}
    \vspace{-15pt}
    \caption{\textbf{Performance of $\bx_t$, $\by_t$, and the EWA of $\by_t$ under varying $\beta_1$.} For suboptimal $\beta_1$, $\by_t$ outperforms $\bx_t$, and across all momentum settings, the EWA of $\by_t$ achieves the lowest loss (\cref{obs:y}).}
    \label{fig:EMA_y}
    \vspace{-5pt}
\end{figure}

\newobservation[obs:y]{In \texttt{SF-AdamW}, the $\by_t$ iterates remain well aligned with the river geometry of the loss landscape, even under suboptimal momentum settings, whereas $\bx_t$ may deviate.
}

\vspace{-5pt}
\subsection{Schedule-Free Methods at the Edge of Stability}\label{sec:eos}
\vspace{-3pt}

We return to the toy model from \cref{sec:toy} to further analyze the optimization dynamics of the $\by_t$ iterates. Notably, the $\by_t$ sequence exhibits a period-two oscillation centered along the river (\cref{fig:toy_model}). 
This behavior resembles the dynamics of full-batch GD at the \emph{Edge of Stability} (EoS), where iterates oscillate along sharp directions while remaining globally stable. The EoS phenomenon was first identified empirically by \citet{cohen2021gradient} and has since been studied theoretically~\citep{arora2022understanding, wang2022analyzing, ahn2023learning, damian2023selfstabilization, song2023trajectory, zhu2023understanding}. A key feature of this regime is that the sharpness---measured by the largest Hessian eigenvalue---stabilizes near the threshold $2/\gamma$.

We now make this connection precise by analyzing SF dynamics through the lens of EoS.

\vspace{-2pt}
\textbf{Notation.}~
Let $f(\bw)$ be the objective and $\mathbf H(\bw)$ its Hessian at $\bw$. For brevity, write $\mathbf H_t \coloneqq \mathbf H(\bw_t)$. The largest eigenvalue $\lambda_1(\mathbf H)$ is called the \emph{sharpness}, and for a given preconditioner $\mathbf P$, $\lambda_1(\mathbf P^{-1}\mathbf H)$ is called the \emph{preconditioned sharpness}.

We first study the stability of Schedule-Free GD (\texttt{SF-GD}) on quadratic objectives, which serve as local Taylor approximations to neural network training losses. \texttt{SF-GD} is defined by \eqref{eq:schefree} with \( \Delta_t \triangleq \nabla f(\by_t) \).

\begin{restatable}[Stability Threshold of \texttt{SF-GD}]{proposition}{stabilityGD}\label{prop:stability_sfgd}
    Consider running \textnormal{\texttt{SF-GD}} on a quadratic objective $f(\bw) = \tfrac{1}{2}\bw^\top \mathbf H \bw + \bg^\top\bw + c$. If $\lambda_1(\mathbf H) > \tfrac{2}{(1-\beta)\gamma}$, then the iterates $\{(\bx_t, \by_t, \bz_t)\}$ diverge.
\end{restatable}
\vspace{-3pt}

Notably, this stability threshold is scaled by a factor of $(1 - \beta)^{-1}$ compared to the threshold for standard GD. This result is consistent with empirical observations from \citet{defazio2024road}, which report that SF methods allow the use of larger LRs, particularly when $\beta$ is close to one.

Next, we extend the analysis to \texttt{SF-PrecondGDW}, defined by \eqref{eq:schefree} with $\Delta_t \triangleq \mathbf P^{-1} \nabla f(\by_t)+\lambda\by_t$, where $\mathbf{P}$ is a fixed, symmetric, and positive definite preconditioner, and $\lambda$ denotes the weight decay coefficient. This setting parallels the analysis of \emph{Adaptive Edge of Stability} in \citet{cohen2022adaptive}.

\begin{restatable}[Stability Threshold of \texttt{SF-PrecondGDW}]{proposition}{stabilityPreconGD}\label{prop:stability_sfpgd}
Consider running \textnormal{\texttt{SF-PrecondGDW}} on $f(\bw) = \tfrac{1}{2}\bw^\top \mathbf H \bw + \bg^\top\bw + c$. If $\lambda_1(\mathbf P^{-1}\mathbf H) > \tfrac{2}{(1-\beta)\gamma}-\lambda$, then the iterates $\{(\bx_t, \by_t, \bz_t)\}$ diverge.
\end{restatable}
\vspace{-3pt}

Proofs of \cref{prop:stability_sfgd,prop:stability_sfpgd} are deferred to \cref{appendix:eos}.

\texttt{SF-AdamW} can be viewed as \texttt{SF-PrecondGDW} with with a slowly varying diagonal preconditioner $\mathbf P_t$. Hence, its stability is governed by the preconditioned sharpness $\lambda_1(\mathbf P_t^{-1}\mathbf H_t)$. As shown in \cref{fig:eos}, the preconditioned sharpness at $\by_t$ iterates equilibrates near the stability threshold in both the toy model and CIFAR-10 experiments---exhibiting a typical EoS behavior.

\begin{figure}[t]
    \vspace{-15pt}
    \begin{subfigure}{0.32\linewidth}
        \centering
        \includegraphics[width=\linewidth]{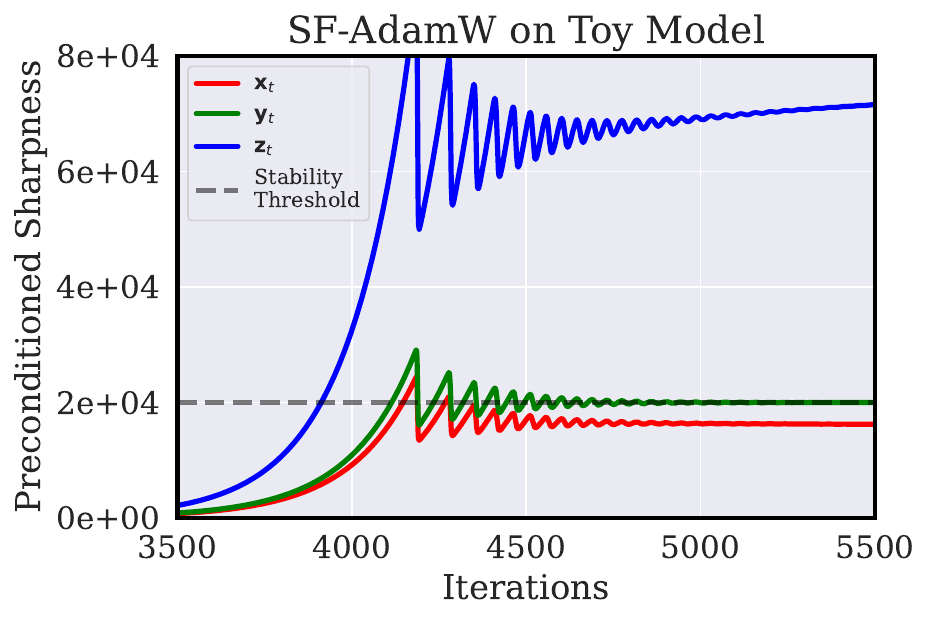}
        
    \end{subfigure}
    \begin{subfigure}{0.3\linewidth}
        \centering
        \includegraphics[width=\linewidth]{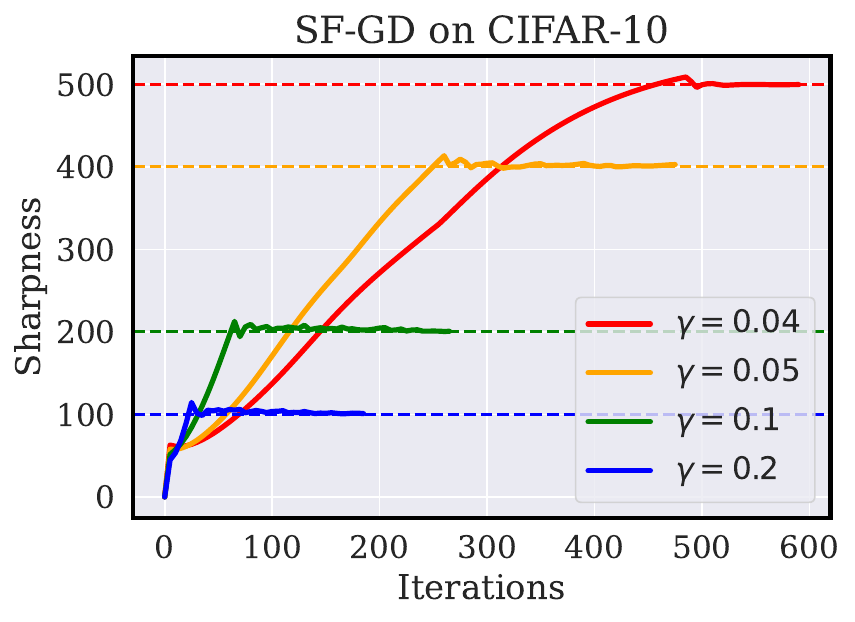}
         
    \end{subfigure}
    \begin{subfigure}{0.3\linewidth}
        \centering
        \includegraphics[width=\linewidth]{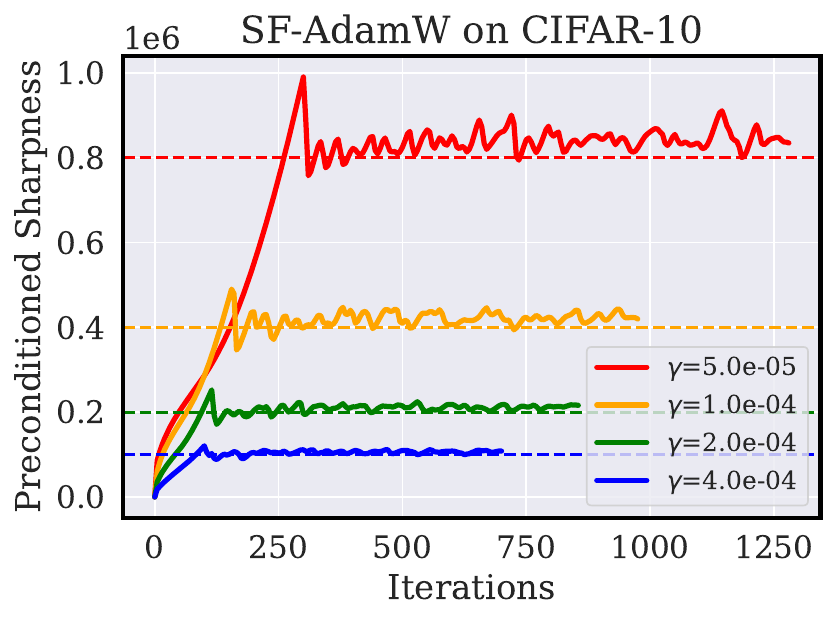}
    \end{subfigure}
    \centering
    \caption{\textbf{The $\by_t$ iterates of Schedule-Free methods operate at the Edge of Stability.}   Plots of (preconditioned) sharpness during full-batch training; dashed lines indicate stability thresholds. \textbf{Left:} Toy model trained using \texttt{SF-AdamW} with $(\beta_1,\beta_2)=(0.9, 0.99)$.  \textbf{Middle, Right:} Fully connected network trained on a 5k subset of CIFAR-10 using \texttt{SF-GD} with $\beta=0.9$ and \texttt{SF-AdamW} with $(\beta_1, \beta_2)=(0.95, 0.99)$; (preconditioned) sharpness is evaluated at the $\mathbf y_t$ iterates (\cref{obs:eos}).}
    \label{fig:eos}
    \vspace{-10pt}
\end{figure}

\newobservation[obs:eos]{In full-batch settings, Schedule-Free methods operate at the Edge of Stability, with the (preconditioned) sharpness at $\by_t$ hovering around the stability threshold.}

\vspace{-5pt}
\subsection{A Reformulation of Schedule-Free Optimizer}
\vspace{-3pt}

Motivated by \cref{obs:y,obs:eos}, we now examine the dynamics of the $\by_t$ iterates.
Following \citet{defazio2024road} and \citet{morwani2025connections},
define the momentum variable
$\mathbf m_t \!\coloneqq\! \tfrac{\mathbf x_t-\mathbf z_{t+1}}{\gamma}$.
The Schedule-Free update in \eqref{eq:schefree} is then equivalent to
\begin{align}
\mathbf m_t &= (1-c_t)\,\mathbf m_{t-1} + \Delta_t, \notag\\
\mathbf y_{t+1} &= \mathbf y_t
      -\gamma\bigl[\beta c_{t+1}\mathbf m_t + (1-\beta)\Delta_t\bigr], \tag{\texttt{SF}$_\by$}\label{eq:schefree_y}
\end{align}
i.e.\ \texttt{SF} is a momentum-based optimizer update on $\mathbf y_t$. The full derivation is provided in \cref{appendix:sf_reformulation}.

The  $\bx_t$ iterate can then be expressed as
\[
\mathbf x_t
  =\frac{(1-c_t)(1-\beta)\mathbf x_{t-1}+c_t\mathbf y_t}
        {(1-c_t)(1-\beta)+c_t}.
\]
In other words, $\bx_t$ is a weighted average of past $\mathbf y_t$’s.  Hence, we arrive at the following conclusion.

\newobservation[obs:sf_reformulation]{Schedule-Free implicitly
performs weight averaging over momentum iterates \emph{without} storing an
extra model copy.}

\vspace{-5pt}
\subsection{Central Flow Analysis}
\vspace{-3pt}

\citet{cohen2025understanding} observe that, at the EoS, the time-averaged optimization trajectory follows a differential equation called the \emph{central flow}, which characterizes the river that the dynamics trace during training.
We adopt this framework to understand the magnitude of oscillation of $\by_t$ iterates of \texttt{SF-AdamW} along the river.   
In particular, we analyze the scalar surrogate \texttt{SF-ScalarAdam} with an adaptive preconditioner $\nu_t$ updated as $\nu_{t} = \beta_2\nu_{t-1}+(1-\beta_2)\|\nabla f(\by_t)\|^2$. 
Based on the reformulated update \eqref{eq:schefree_y}, and assuming that $c(t) = 1/t$ becomes negligible for sufficiently large $t$, the central flow equations are given by:
\[
\frac{d\mathbf y}{dt}= -\frac{\gamma(1-\beta_1)}{\sqrt{\nu}}
\Bigl[\nabla f(\mathbf y)+\tfrac12\sigma^2\nabla S(\mathbf y)\Bigr],
\qquad
\frac{d\nu}{dt}= \frac{1-\beta_2}{\beta_2}\!
\Bigl[\|\nabla f(\mathbf y)\|^2+\sigma^2S(\mathbf y)^2-\nu\Bigr],
\]
where $S(\mathbf y)=\lambda_{1}\!\bigl(\nabla^{2}f(\mathbf y)\bigr)$ and
$\sigma^2$ is the steady-state variance of oscillations along the hill
direction. Enforcing the stability condition  
$S(\mathbf y)/\sqrt{\nu}=2/[(1-\beta_1)\gamma]$ yields
\[
\sigma^2(\mathbf y)=
\frac{\langle\nabla S,-\nabla f\rangle+\frac{1-\beta_2}{\beta_2}\Bigl[\frac14
S^2-\frac{\|\nabla f\|^2}{(1-\beta_1)^2\gamma^2}\Bigr]}
     {\frac12\|\nabla S\|^2+\frac{1-\beta_2}{(1-\beta_1)^2\beta_2\gamma^2}S^2}.
\]
As $\beta_1$ increases, $\sigma^2$ decreases; thus, larger values of $\beta_1$ suppress oscillations along the hill directions, keeping the $\mathbf y_t$ iterates more closely aligned with the river—consistent with the empirical observation in \cref{fig:toy_model}. A complete derivation, including the central flow of \texttt{SF-GD}, is provided in \cref{appendix:central_flow}.

\vspace{-5pt}
\section{A Refined and Robust Schedule-Free Optimizer} \label{sec:refined}
\vspace{-5pt}

\begin{figure}[b]
    \vspace{-15pt}
    \centering
    \includegraphics[width=\linewidth]{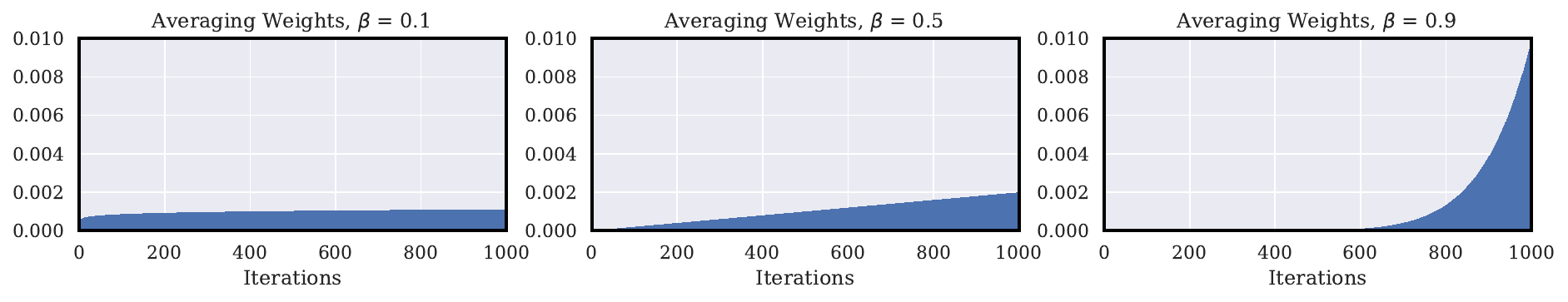}
    \vspace{-15pt}
    \caption{\textbf{Averaging weights in SF method.} Smaller values of $\beta$ flatten the averaging weights $\{\alpha_t\}$.}
    \label{fig:weights_histogram}
    \vspace{-15pt}
\end{figure}

While \texttt{SF-AdamW} achieves strong performance, it is highly sensitive to momentum hyperparameters and degrades under large batch sizes~\citep{zhang2025how, morwani2025connections}. 
Building on the insights from \cref{sec:sf_understand}, we revisit these issues and propose a refined variant that addresses them. 

A key limitation in the vanilla \eqref{eq:schefree} setup ($c_t=1/t$) is that $\beta$ simultaneously controls both (i) the momentum applied to $\mathbf y_t$~\eqref{eq:schefree_y} and (ii) the implicit averaging window that defines $\mathbf x_t$:
\[
\mathbf x_T
  =\sum_{t=1}^{T}\!\alpha_t\,\mathbf y_t,\quad
  \alpha_t:=\frac{c_t}{(1-c_t)(1-\beta)+c_t}\!\prod_{s=t+1}^{T}\!\left[\frac{(1-c_s)(1-\beta)}{(1-c_s)(1-\beta)+c_s}\right].
\]
When $\beta$ is small, the weights $\{\alpha_t\}$ become stretched, overemphasizing early iterates and preventing \(\mathbf x_t\) from closely tracking the river, as shown in \cref{fig:weights_histogram}. This also explains \cref{obs:sf_river}, where we demonstrated that applying EWA to $\bx_t$ offers no benefit: since $\bx_t$ is already a weighted average of $\by_t$, further averaging merely flattens the weights and weakens alignment with the river.

More fundamentally, $\beta$ plays a \emph{dual role}: it controls both the momentum update of $\by_t$ and the width of the averaging window for $\bx_t$. The optimal value for each may differ, and this mismatch can hinder performance. 
In large-batch training, a narrower averaging window (i.e., larger $\beta$) is preferred to emphasize recent iterates. However, a large $\beta$ also slows the update of $\by_t$, as the $(1 - \beta)\Delta_t$ term in \eqref{eq:schefree_y} becomes small, reducing the influence of recent gradients.

\paragraph{Our Refined SF Method.}
We introduce an additional \emph{decoupling} parameter $C$ and redefine $c_t = 1/t$ in \eqref{eq:schefree} as
\[
c_t=\frac{(1-\beta)C}{t},
\quad\textrm { which then leads to }\quad
\alpha_t\;\approx\;\frac{C}{T}\Bigl(\frac{t}{T}\Bigr)^{C-1}.
\]
The full derivation is given in \cref{appendix:refined}, and pseudocode is provided in \cref{alg:sf_adamw_refined}. As shown in \cref{fig:modified_weights_histogram,fig:modified_weights_histogram_T5000}, the weights $\{\alpha_t\}$ closely follow the theoretical approximation across different values of $\beta$, $C$, and $T$. Notably, both vanilla and refined SF use an averaging window that widens with $T$, unlike fixed-width schemes such as EWA.

This modification makes the averaging weights $\{\alpha_t\}$ depend \emph{solely} on $C$, allowing $\beta$ to independently control the momentum on $\mathbf y_t$. In other words, $C$ \emph{decouples} the momentum and averaging behavior. Below, we present preliminary experiments demonstrating the benefits of this refinement.

\vspace{-3pt}
\paragraph{Empirical gains.}
We evaluate the performance of our refinement to \texttt{SF-AdamW} using the experimental setup as in \cref{sec:sf_critiques}. Keeping all other hyperparameters fixed as the tuned vanilla configuration and varying only $C$ (with $C = 1/(1 - \beta_1)$ recovering the original method), we observe:
\begin{itemize}
    \item \textbf{Momentum robustness. }  
  For $\beta_1\in\{0.1,\,0.5\}$, the refined
  \texttt{SF-AdamW} allows $\bx_t$ to match or outperform $\by_t$, in contrast to the underperformance of $\bx_t$ reported in \cref{obs:y}~(\cref{fig:modified}-left).
\item  \textbf{Improved best-case performance. }  
  In the best vanilla setup ($\beta_1 = 0.95$), setting $C = 200$ leads to further reductions in validation loss; similar gains are observed with $\beta_1 = 0.9$~(\cref{fig:modified}-middle).
\item  \textbf{Large-batch setting. }  
  With 2M-token batches, vanilla \texttt{SF-AdamW} ($\beta_1=0.98$) lags behind \texttt{AdamW} with cosine schedule,  whereas setting $C = 200$ matches its final performance~(\cref{fig:modified}-right).
\end{itemize}

\looseness=-1
To further examine the sensitivity to the new hyperparameter $C$, we conduct a set of controlled sweeps (\cref{tab:c_sweep_05M,tab:c_sweep_2M} in \cref{appendix:refined_sensitivity}). Across both optimal and suboptimal $\beta_1$, we find that refined \texttt{SF-AdamW} consistently outperforms the vanilla baseline over a wide range of $C$, demonstraing its robustness to this parameter. These results show that decoupling momentum and averaging via introducing $C$ eliminates the tradeoff inherent in vanilla SF, yielding a more robust and scalable optimizer.

\begin{figure}[t]
    \vspace{-15pt}
    \begin{subfigure}{0.32\linewidth}
        \centering
        \includegraphics[width=\linewidth]{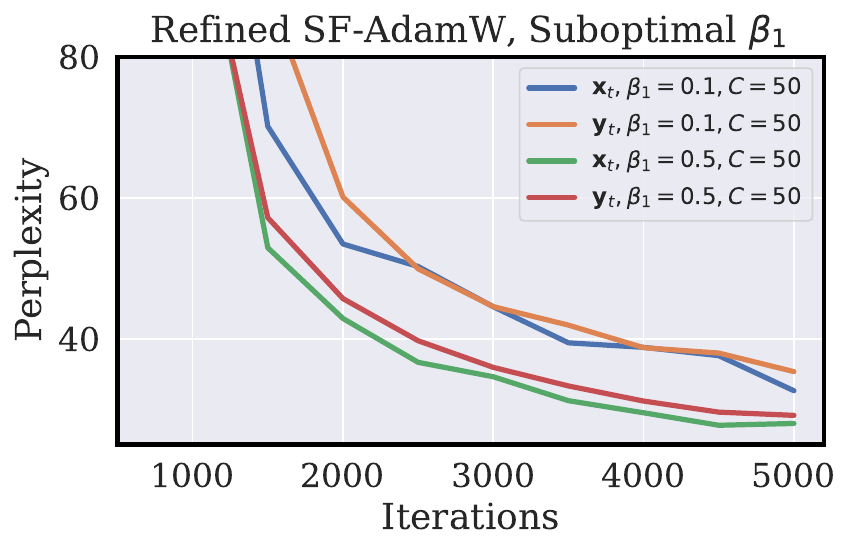}
        \label{fig:modified_robust}
    \end{subfigure}
    \hfill
    \begin{subfigure}{0.32\linewidth}
        \centering
        \includegraphics[width=\linewidth]{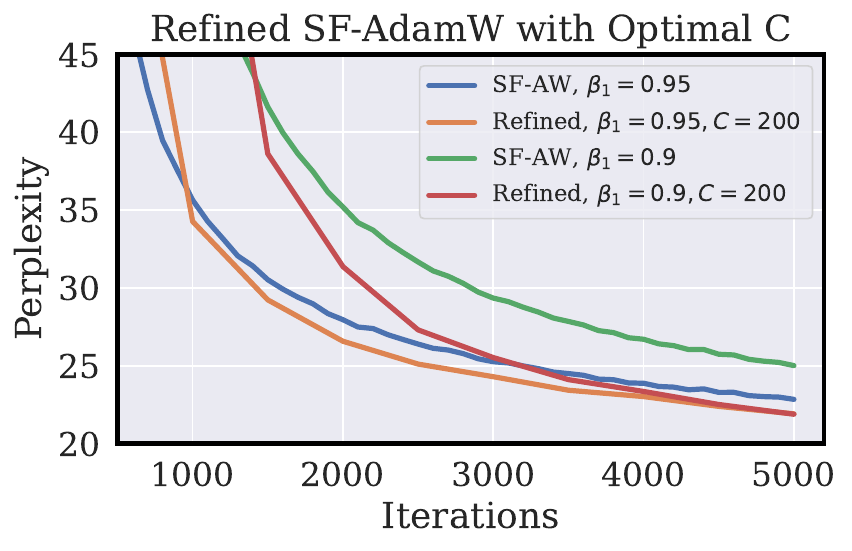}
        \label{fig:modified_improve}
    \end{subfigure}
    \hfill
    \begin{subfigure}{0.32\linewidth}
        \centering
        \includegraphics[width=\linewidth]{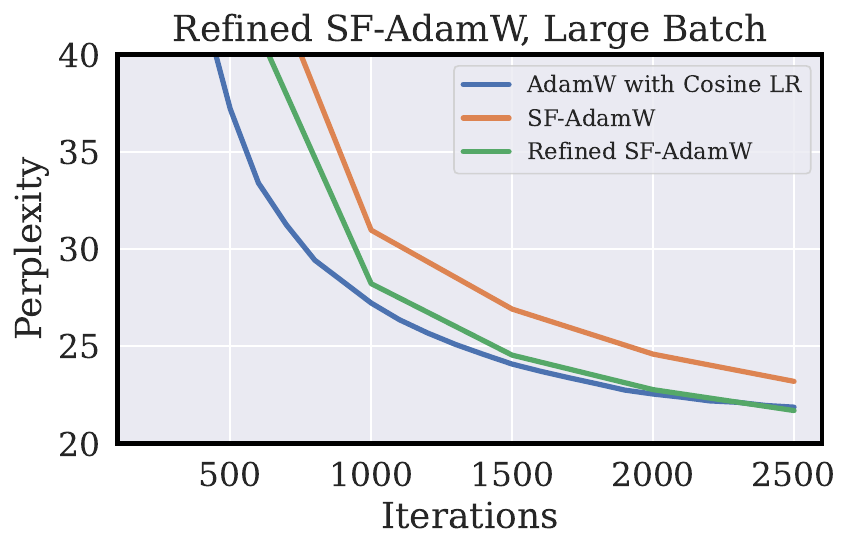}
        \label{fig:modified_large}
    \end{subfigure}
    \caption{\textbf{Refined SF-AdamW.} \textbf{Left:} Performance of $\bx_t$ and $\by_t$ iterates using the refined \texttt{SF-AdamW} with $\beta_1 \in \{0.1, 0.5\}$ and $C = 50$. 
    \textbf{Middle:} Refined \texttt{SF-AdamW} with $\beta_1 \in \{0.95, 0.9\}$ and $C = 200$ achieves improved performance over the best vanilla \texttt{SF-AdamW} run. \textbf{Right:} Under a large batch size (2M tokens), vanilla \texttt{SF-AdamW} with $\beta_1=0.98$ underperforms compared to \texttt{AdamW} with a cosine schedule, while the refined \texttt{SF-AdamW}---with only a sweep over $C = 200$---matches its final performance.}
    \label{fig:modified}
    \vspace{-10pt}
\end{figure}

\vspace{-5pt}
\section{Conclusion}\label{sec:conclusion}
\vspace{-5pt}

We presented a principled view of Schedule-Free (SF) methods by studying its behavior through the geometry of the river-valley loss landscape. Our analysis shows that SF methods naturally follow the river without requiring explicit LR decay or weight averaging, making them a compelling approach for scalable pretraining. We further provided theoretical insights grounded in Edge of Stability and central flow dynamics. Building on this understanding, we proposed a refined variant that decouples momentum from averaging, improving both robustness and performance.

While our findings highlight the potential of SF methods, several open questions remain. Our theoretical analysis relies on simplifying assumptions, and validating the central flow approximation in deep learning is a natural next step. Furthermore, extending the river-valley framework to analyze other modern optimizers, as well as exploring their integration with SF methods, is a promising direction for further investigation. Lastly, our experiments are limited to small-scale language models due to computational constraints. Scaling these findings to larger models and longer training durations remains an important direction for future work.

\begin{ack}
This work was supported by a National Research Foundation of Korea (NRF) grant funded by the Korean government (MSIT) (No.\ RS-2024-00421203) and the InnoCORE program of the Ministry of Science and ICT (No.\ N10250156).
\end{ack}

\bibliography{refs}
\bibliographystyle{plainnat}


\newpage
\appendix

\renewcommand{\appendixpagename}{\centering \LARGE Appendix}
\appendixpage

\startcontents[section]
\printcontents[section]{l}{1}{\setcounter{tocdepth}{2}}

\section*{Impact Statement}\label{sec:impact}
This work studies scalable strategies for language model (LM) pretraining, with a focus on foundational algorithmic design. While the improper application of these algorithms in downstream tasks may lead to LMs producing harmful, offensive, or privacy-violating content, such applications fall outside the scope of this paper. Our contribution is limited to understanding and improving the pretraining algorithm itself.

\section{Pseudocode for Schedule-Free AdamW}

For completeness, we present the pseudocode for the Schedule-Free AdamW (\texttt{SF-AdamW}) algorithm in \cref{alg:sf_adamw}, along with our proposed refinement in \cref{alg:sf_adamw_refined}. The refinement introduces a decoupling parameter $C$ to independently control the averaging window, addressing the coupling issue discussed in \cref{sec:refined}.

\begin{algorithm}[htb]
\begin{algorithmic}[1]
    \STATE {\bfseries Input:} $x_1, \text{learning rate } \gamma, \text{decay } \lambda, \text{warmup steps } T_{\text{warmup}}, \beta_1, \beta_2, \epsilon$
    \STATE $z_1 = x_1$
    \STATE $v_0 = 0$
    \FOR{$t=1$ {\bfseries to} $T$}

    \STATE $y_t = (1-\beta_1)z_{t} + \beta_1 x_{t}$
    \STATE $g_{t} \in \partial f(y_t, \zeta_t)$
    \STATE $v_t = \beta_2 v_{t-1} + (1-\beta_2) g_t^2$
    \STATE $\hat{v}_t = v_t /(1-\beta_2^t)$
    \STATE $\gamma_t = \gamma \min(1, t/T_{\text{warmup}})$
    \STATE $z_{t+1} = z_{t}-\gamma_t g_t/(\sqrt{\hat{v}_t}+\epsilon) - \gamma_t \lambda y_t$
    \STATE $c_{t+1} = \frac{\gamma^2_t}{\sum^t_{i=1} \gamma^2_i}$
    \STATE $x_{t+1} = \left(1-c_{t+1} \right) x_t + c_{t+1}z_{t+1}$
    \ENDFOR
    \STATE Return $x_{T+1}$
\end{algorithmic}
\caption{\texttt{SF-AdamW}~\citep{defazio2024road}}\label{alg:sf_adamw}
\end{algorithm}

\begin{algorithm}[htb]
\begin{algorithmic}[1]
    \STATE {\bfseries Input:} $x_1, \text{learning rate } \gamma, \text{decay } \lambda, \text{warmup steps } T_{\text{warmup}}, \beta_1, \beta_2, \epsilon$, \textcolor{magenta}{decoupling parameter $C$}
    \STATE $z_1 = x_1$
    \STATE $v_0 = 0$
    \FOR{$t=1$ {\bfseries to} $T$}

    \STATE $y_t = (1-\beta_1)z_{t} + \beta_1 x_{t}$
    \STATE $g_{t} \in \partial f(y_t, \zeta_t)$
    \STATE $v_t = \beta_2 v_{t-1} + (1-\beta_2) g_t^2$
    \STATE $\hat{v}_t = v_t /(1-\beta_2^t)$
    \STATE $\gamma_t = \gamma \min(1, t/T_{\text{warmup}})$
    \STATE $z_{t+1} = z_{t}-\gamma_t g_t/(\sqrt{\hat{v}_t}+\epsilon) - \gamma_t \lambda y_t$
    \STATE $c_{t+1} =  \textcolor{magenta}{\min} \left\{\frac{\gamma^2_t}{\sum^t_{i=1} \gamma^2_i}  \textcolor{magenta}{\cdot(1-\beta_1)C,\; 1} \right\}$
    \STATE $x_{t+1} = \left(1-c_{t+1} \right) x_t + c_{t+1}z_{t+1}$
    \ENDFOR
    \STATE Return $x_{T+1}$
\end{algorithmic}
\caption{Refined \texttt{SF-AdamW} (with decoupling parameter $C$)}\label{alg:sf_adamw_refined}
\end{algorithm}

\section{Experimental Details}\label{appendix:experiment}
\subsection{Language Model Experiments}

\paragraph{Codebase.}
All language model experiments are implemented using the public \texttt{llm-baselines} codebase: \url{https://github.com/epfml/llm-baselines}. Our only modification is the addition of custom optimizer implementations to support the Schedule-Free and refined Schedule-Free methods. All other components (model, data pipeline, logging) remain unchanged.

\paragraph{Architectures.}
Our main experiments use a 124M-parameter LLaMA-style decoder-only transformer with SwiGLU activations~\citep{shazeer2020glu}, RoPE embeddings~\citep{su2024roformer}, RMSNorm~\citep{zhang2019root}, and alternating attention/MLP blocks (12 layers, 12 attention heads, hidden dimension 768). Additional results in \cref{appendix:additional_results} verify our findings with a 124M-parameter GPT-2-style transformer. Both architectures are implemented in \texttt{llm-baselines} with standard design choices.

\paragraph{Datasets.}

Our main experiments use the 6B-token subset of the SlimPajama dataset~\citep{cerebras2023slimpajama}, available on Hugging Face.\footnote{\url{https://huggingface.co/datasets/DKYoon/SlimPajama-6B}} We tokenize with the GPT-2 tokenizer~\citep{radford2019language}, which has a vocabulary size of 50,304. In \cref{appendix:additional_results}, we also validate our main findings on the OpenWebText2 dataset~\citep{pile},\footnote{\url{https://huggingface.co/datasets/segyges/OpenWebText2}} using the same setup.

\paragraph{Training details.}
We train models using \texttt{AdamW} and \texttt{SF-AdamW}, with a short warmup phase comprising 5\% of total steps. For large-batch runs with cosine decay, the learning rate is annealed to 10\% of its peak. Main experiments use a batch size of 1,024 sequences of context length 512 tokens (0.5M tokens total), trained for 5,000 steps, amounting to roughly 2.5B tokens ($\sim 1\times$ Chinchilla scale). Large-batch experiments use a 2M-token batch size and 2,500 steps ($\sim$5B tokens, or $\sim 2\times$ Chinchilla scale) to evaluate efficiency in the overtrained regime. Validation is performed during training using 3,200 sequences of context length 512 tokens ($\sim$1.6M tokens) to compute validation loss (perplexity) curves. For computing EWA, we use a decay factor of 0.99 for all experiments.

\paragraph{Hyperparameters.}
We fix the weight decay to 0.1 in all experiments. Gradient clipping is set to 1.0 for \texttt{AdamW} and disabled (0.0) for \texttt{SF-AdamW}. We perform sweeps over the learning rate, momentum parameters, and the decoupling parameter $C$ (for refined \texttt{SF-AdamW}). Full configurations are provided in \cref{tab:hyperparameter_main,tab:hyperparameter_main_largeBS,tab:hyperparameter_owt,tab:hyperparameter_owt_largeBS}.

\paragraph{Compute Resources.}
All experiments are conducted on a single node with 8 NVIDIA A6000 GPUs (48GB VRAM each) using data-parallel training. A typical full 5,000-step run with a 0.5M-token batch size takes approximately 3 hours.

\begin{table}[!htp]
    \centering
    \small
    \begin{tabular}{@{}r|lll@{}}
        \toprule
        Optimizer & Learning Rate & $(\beta_1, \beta_2)$ & $C$ (Refined \texttt{SF}) \\
        \midrule
        \texttt{AdamW} & \{5e-4, \textbf{1e-3}, 2e-3, 5e-3\} & \{\textbf{(0.9, 0.95)}, (0.95, 0.99)\} & -- \\
        \texttt{SF-AdamW} & \{1e-3, \textbf{2e-3}, 5e-3\} & \{(0.9, 0.99), \textbf{(0.95, 0.99)}, (0.98, 0.99)\} & -- \\
        Refined \texttt{SF-AdamW} & 2e-3 & (0.95, 0.99) & \{20, 50, 100, \textbf{200}, 500\} \\
        \bottomrule
    \end{tabular}
    \vspace{5pt}
    \caption{\textbf{Hyperparameter sweep: Main experiments.} 
    Grid of hyperparameters used in our main experiments (SlimPajama-6B, 0.5M-token batch size), including learning rates, momentum pairs $(\beta_1, \beta_2)$, and the decoupling parameter $C$ (for Refined \texttt{SF-AdamW}). Bold entries indicate the best-performing configuration for each optimizer.}
    \label{tab:hyperparameter_main}
\end{table}

\begin{table}[!htp]
    \centering
    \small
    \begin{tabular}{@{}r|lll@{}}
        \toprule
        Optimizer & Learning Rate & $(\beta_1, \beta_2)$ & $C$ (Refined \texttt{SF}) \\
        \midrule
        \texttt{AdamW} (Cosine LR) & \{5e-4, \textbf{1e-3}, 2e-3, 5e-3\} & \{\textbf{(0.9, 0.95)}, (0.95, 0.99)\} & -- \\
        \texttt{SF-AdamW} & \{1e-3, \textbf{2e-3}, 5e-3\} & \{(0.9, 0.99), (0.95, 0.99), \textbf{(0.98, 0.99)}\} & -- \\
        Refined \texttt{SF-AdamW} & 2e-3 & (0.98, 0.99) & \{20, 50, 100, \textbf{200}, 500\} \\
        \bottomrule
    \end{tabular}
    \vspace{5pt}
    \caption{\textbf{Hyperparameter sweep: Large-batch experiments.} 
    Grid of hyperparameters used in our large-batch experiments (SlimPajama-6B, 2M-token batch size), including learning rates, momentum pairs $(\beta_1, \beta_2)$, and the decoupling parameter $C$ (for Refined \texttt{SF-AdamW}). Bold entries indicate the best-performing configuration.}
    \label{tab:hyperparameter_main_largeBS}
\end{table}

\begin{table}[!htp]
    \centering
    \small
    \begin{tabular}{@{}r|lll@{}}
        \toprule
        Optimizer & Learning Rate & $(\beta_1, \beta_2)$ & $C$ (Refined \texttt{SF}) \\
        \midrule
        \texttt{AdamW} & \{5e-4, \textbf{1e-3}, 2e-3\} & \{\textbf{(0.9, 0.95)}, (0.95, 0.99)\} & -- \\
        \texttt{SF-AdamW} & \{1e-3, \textbf{2e-3}, 5e-3\} & \{(0.9, 0.99), \textbf{(0.95, 0.99)}, (0.98, 0.99)\} & -- \\
        Refined \texttt{SF-AdamW} & 2e-3 & (0.95, 0.99) & \{50, 100, \textbf{200}, 500\} \\
        \bottomrule
    \end{tabular}
    \vspace{5pt}
    \caption{\textbf{Hyperparameter sweep: OpenWebText2 experiments.} 
    Grid of hyperparameters used in our additional experiments on OpenWebText2 (0.5M-token batch size), including learning rates, momentum pairs $(\beta_1, \beta_2)$, and the decoupling parameter $C$ (for Refined \texttt{SF-AdamW}). Bold entries indicate the best-performing configuration.}
    \label{tab:hyperparameter_owt}
\end{table}

\begin{table}[!htp]
    \centering
    \small
    \begin{tabular}{@{}r|lll@{}}
        \toprule
        Optimizer & Learning Rate & $(\beta_1, \beta_2)$ & $C$ (Refined \texttt{SF}) \\
        \midrule
        \texttt{AdamW} (Cosine LR) & \{5e-4, 1e-3, 2e-3, \textbf{5e-3}\} & \{\textbf{(0.9, 0.95)}, (0.95, 0.99)\} & -- \\
        \texttt{SF-AdamW} & \{5e-4, 1e-3, \textbf{2e-3}, 5e-3\} & \{(0.95, 0.99), \textbf{(0.98, 0.99)}\} & -- \\
        Refined \texttt{SF-AdamW} & \{2e-3, \textbf{5e-3}, 1e-2\} & (0.98, 0.99) & \{20, 50, 100, 200, \textbf{500}, 1000\} \\
        \bottomrule
    \end{tabular}
    \vspace{5pt}
    \caption{\textbf{Hyperparameter sweep: OpenWebText2 large-batch experiments.} 
    Grid of hyperparameters used in our additional large-batch experiments on OpenWebText2 (2M-token batch size), including learning rates, momentum pairs $(\beta_1, \beta_2)$, and the decoupling parameter $C$ (for Refined \texttt{SF-AdamW}). Bold entries indicate the best-performing configuration.}
    \label{tab:hyperparameter_owt_largeBS}
\end{table}

\subsection{Edge of Stability Experiments on CIFAR-10}

We provide the experimental setup for \cref{fig:eos} in \cref{sec:eos}, where we study whether Schedule-Free methods operate at the Edge of Stability in a deep learning setting.

Our experiments build on the public \texttt{edge-of-stability} codebase,\footnote{\url{https://github.com/locuslab/edge-of-stability}} modifying only the optimizer to incorporate Schedule-Free methods. The model is a 3-layer MLP with hidden width 200 and $\tanh$ activations, trained on the first 5,000 samples of CIFAR-10 using mean squared error (MSE) loss.

For \texttt{SF-GD}, we fix momentum $\beta = 0.9$ and vary the learning rate, training each run until the loss reaches 0.02.
For \texttt{SF-AdamW}, we fix $\beta_1 = 0.95$, $\beta_2 = 0.99$, and weight decay at 0.1, and vary the learning rate, training until the loss reaches 0.05.

\section{Additional Results}\label{appendix:additional_results}

In this section, we present additional experiments on the OpenWebText2 dataset using a 124M-parameter GPT-2 style decoder-only transformer, replicating the setups from the main text. As summarized below, the results confirm that our main findings~(\cref{obs:sf_river,obs:sensitive-momentum,obs:y}) generalize across both datasets and architectures.

\begin{itemize}
    \item \textbf{\cref{obs:sf_river}: } 
    As in the main experiments, we perform a grid search to determine the best hyperparameters for both \texttt{AdamW} and \texttt{SF-AdamW}, and evaluate whether LR decay or EWQ improves performance. In this setup, the best-performing configuration for \texttt{AdamW} is $(\beta_1, \beta_2) = (0.9, 0.95)$ with LR 1e-3, while \texttt{SF-AdamW} achieves optimal results with $(\beta_1, \beta_2) = (0.95, 0.99)$ and LR 2e-3. We observe that neither LR decay nor EWA improves performance for \texttt{SF-AdamW}, indicating that it already tracks the river closely (see~\cref{fig:Obs1_Openwebtext2}).
    \item \textbf{\cref{obs:sensitive-momentum}: } 
    We also assess the behavior of \texttt{SF-AdamW} under suboptimal momentum configurations, using $\beta_1 \in \{0.1, 0.5\}$ while keeping all other settings fixed. For both values, we observe that applying a short LR decay phase using \texttt{AdamW} significantly improves performance, reducing the validation loss compared to the constant LR baseline. This result shows that \texttt{SF-AdamW} is sensitive to the choice of momentum, and suboptimal settings can hinder its ability to effectively follow the river (see~\cref{fig:Decay_AdamW_betas_openwebtext2}).
    \item \textbf{\cref{obs:y}: } 
    Using the same runs from above, we compute the validation loss at the $\by_t$ iterates and their EWA. For suboptimal momentum settings ($\beta_1 \in \{0.1, 0.5\}$), the loss at $\by_t$ is consistently lower than that at $\bx_t$, consistent with trends observed in both the toy model and SlimPajama experiments. Moreover, across all momentum values, applying EWA to $\by_t$ further improves performance, suggesting that $\by_t$ is more robust to suboptimal momentum and remains better aligned with the river trajectory (see~\cref{fig:EMA_y_openwebtext2}).
    \item \textbf{Refined SF-AdamW: } 
    We evaluate the refined variant of \texttt{SF-AdamW} by sweeping over $C$ across multiple momentum settings. In the low-momentum regime ($\beta_1 = 0.5$), the refined method enables $\bx_t$ to match or outperform $\by_t$, addressing the discrepancy observed in the vanilla formulation. In the best-performing vanilla configuration ($\beta_1 = 0.95$, LR 2e-3), setting $C = 200$ yields further performance gains. Notably, comparable results can also be achieved under a suboptimal momentum setting ($\beta_1 = 0.9$) by choosing $C = 50$, demonstrating improved robustness to hyperparameter choices. In the large-batch setting (2M-token batches), vanilla \texttt{SF-AdamW} with a constant LR ($(\beta_1, \beta_2)=(0.98, 0.99)$, LR 2e-3) underperforms relative to \texttt{AdamW} with a cosine LR schedule ($(\beta_1, \beta_2)=(0.9, 0.95)$, LR 5e-3). However, the refined variant with $C = 500$ ($(\beta_1, \beta_2)=(0.98, 0.99)$, LR 5e-3) successfully closes this performance gap. Together, this suggest that our refined method is robust to momentum and batch size scaling (see~\cref{fig:modified_openwebtext2}).
\end{itemize}

\begin{figure}[htp]
    \begin{subfigure}{0.33\linewidth}
        \centering
        \includegraphics[width=\linewidth]{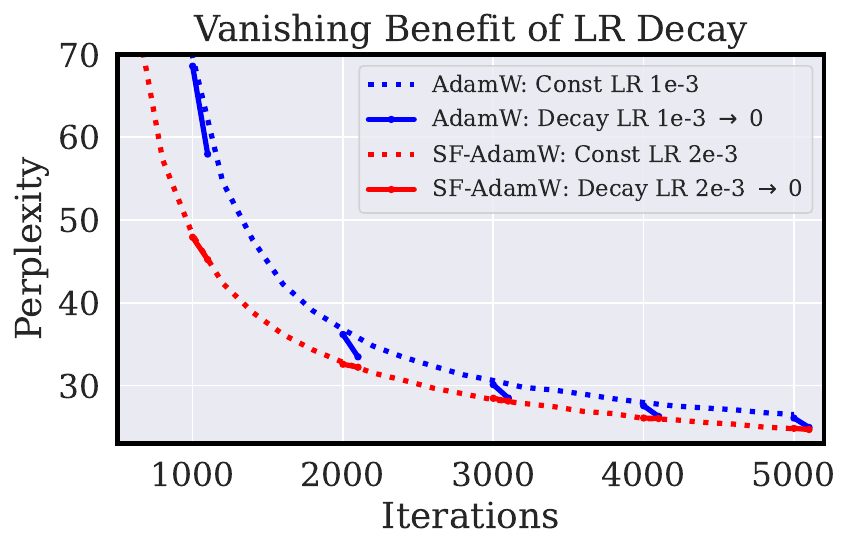}
    \end{subfigure}
    \begin{subfigure}{0.33\linewidth}
        \centering
        \includegraphics[width=\linewidth]{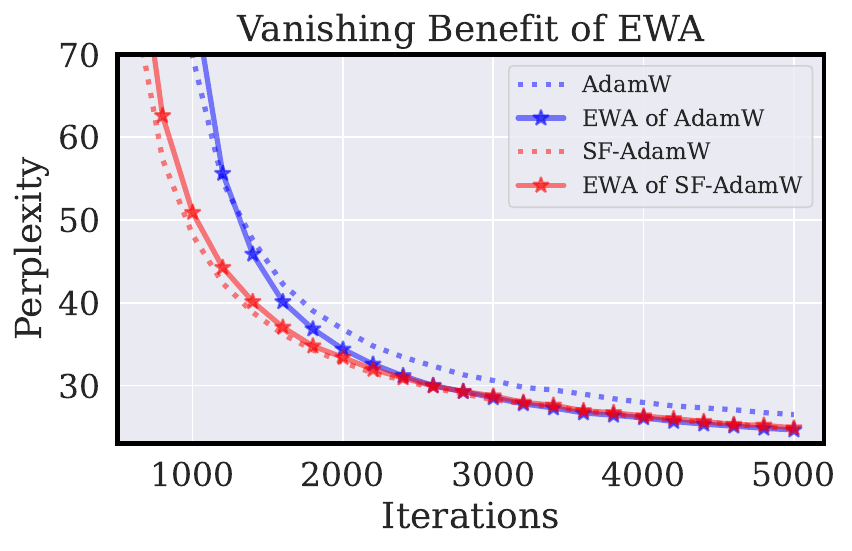}
    \end{subfigure}
    \begin{subfigure}{0.33\linewidth}
        \centering
        \includegraphics[width=\linewidth]{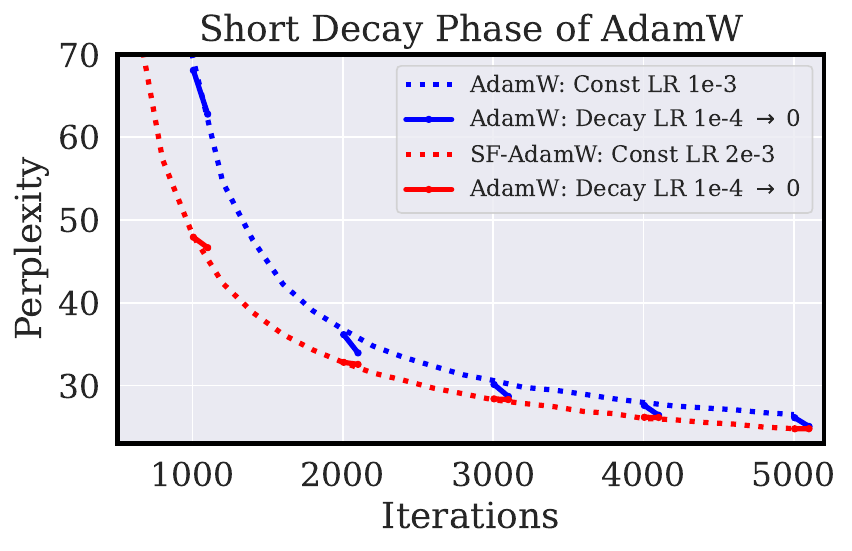}
    \end{subfigure}
    \caption{\textbf{OpenWebText2 Experiment: SF-AdamW closely follows the river, unlike AdamW.} \textbf{Left, Middle:} While \texttt{AdamW} benefits from linear LR decay and EWA, \texttt{SF-AdamW} shows no improvement from either. \textbf{Right:} A short decay phase of \texttt{AdamW} (with linear LR decay from 1e-4 to 0) leads to a sharp loss drop for \texttt{AdamW}, but has no effect when applied to the \texttt{SF-AdamW} trajectory---suggesting that \texttt{SF-AdamW} already tracks the river throughout training (\cref{obs:sf_river}).}
    \label{fig:Obs1_Openwebtext2}
\end{figure}

\begin{figure}[htp]
    \centering
    \includegraphics[width=0.5\linewidth]{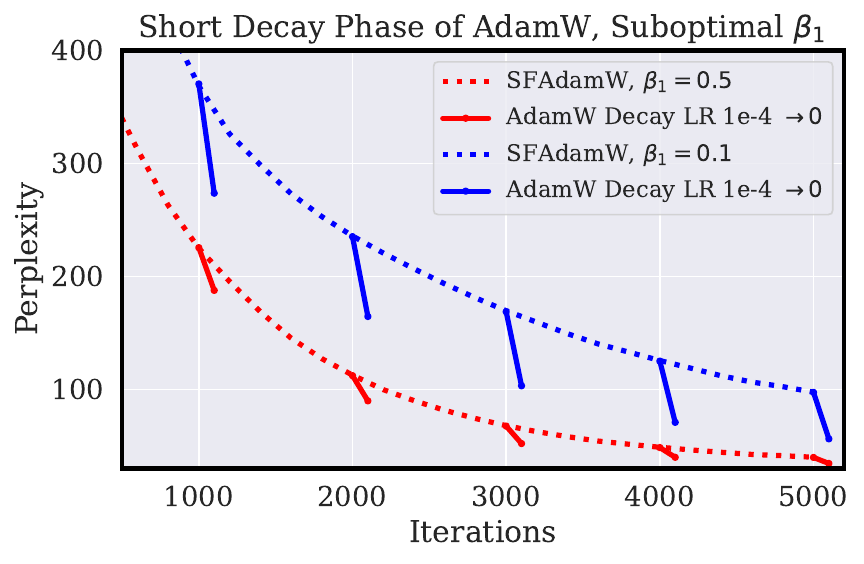}
    \caption{\textbf{OpenWebText2 Experiment: SF-AdamW with suboptimal momentum fails to follow the river.}  A short decay phase of \texttt{AdamW} applied to \texttt{SF-AdamW} checkpoints with $\beta_1 \in \{0.1, 0.5\}$ results in a sharp loss drop, unlike the case with $\beta_1 = 0.95$ (\cref{obs:sensitive-momentum}).}
    \label{fig:Decay_AdamW_betas_openwebtext2}
\end{figure}

\begin{figure}[htp]
    \centering
    \includegraphics[width=\linewidth]{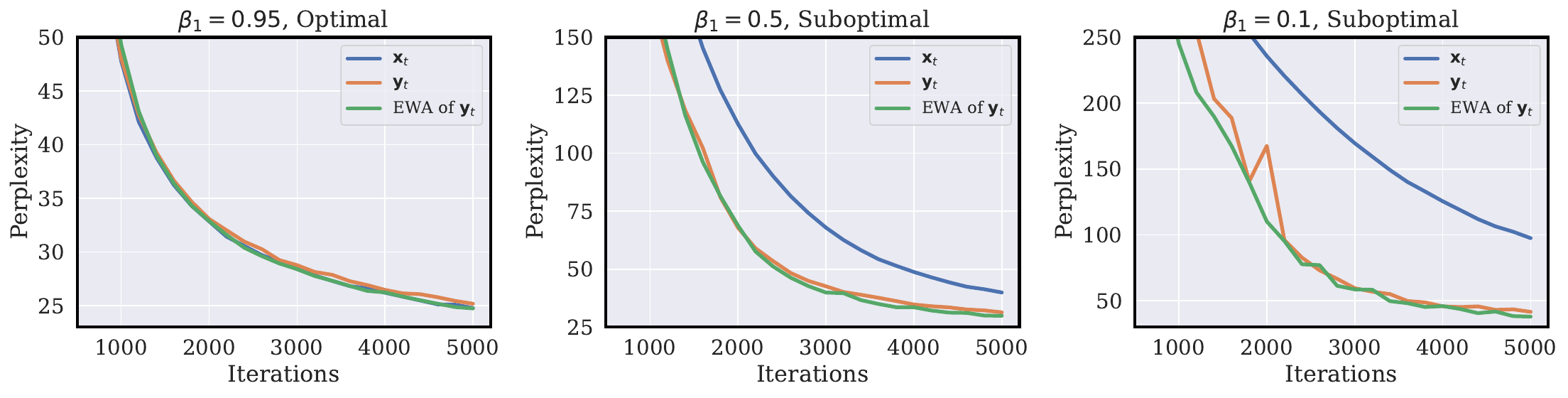}
    \caption{\textbf{OpenWebText2 Experiment: Performance of $\bx_t$, $\by_t$, and the EWA of $\by_t$ under varying $\beta_1$.} For suboptimal $\beta_1$, $\by_t$ outperforms $\bx_t$, and across all momentum settings, the EWA of $\by_t$ achieves the lowest loss (\cref{obs:y}).}
    \label{fig:EMA_y_openwebtext2}
\end{figure}

\begin{figure}[htp]
    \begin{subfigure}{0.32\linewidth}
        \centering
        \includegraphics[width=\linewidth]{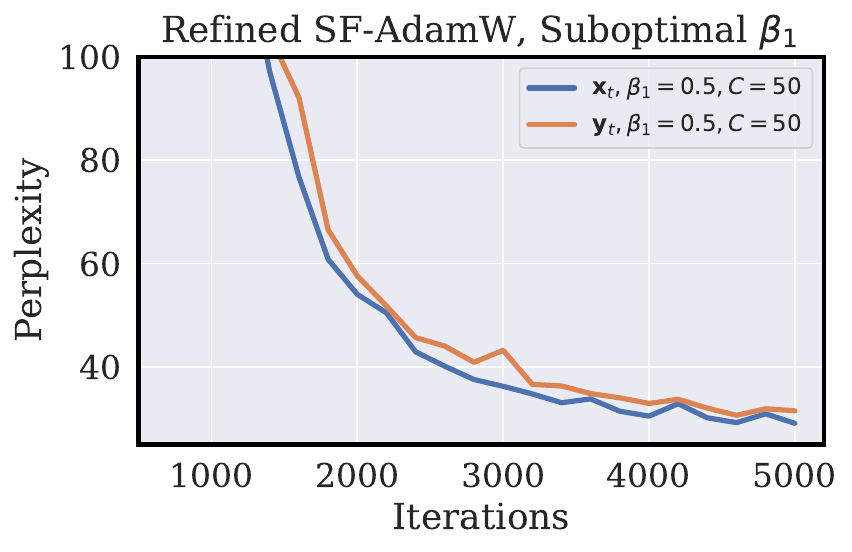}
    \end{subfigure}
    \hfill
    \begin{subfigure}{0.32\linewidth}
        \centering
        \includegraphics[width=\linewidth]{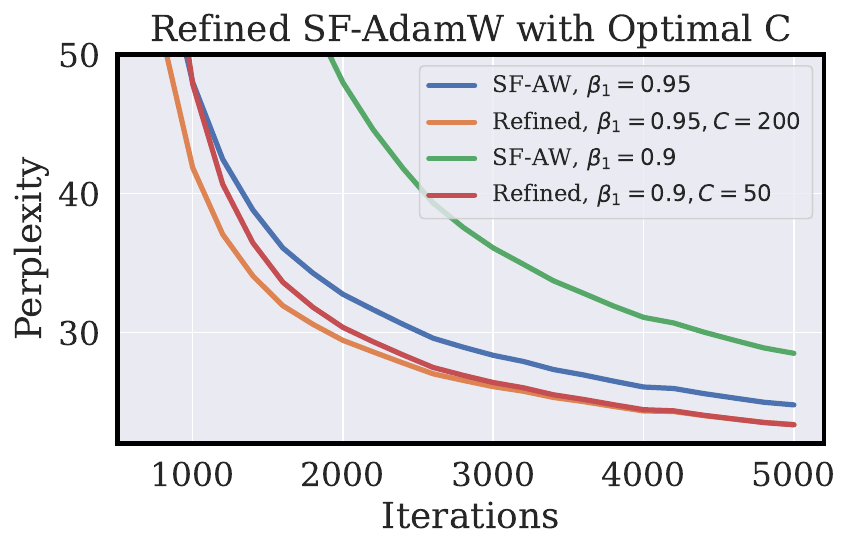}
    \end{subfigure}
    \hfill
    \begin{subfigure}{0.32\linewidth}
        \centering
        \includegraphics[width=\linewidth]{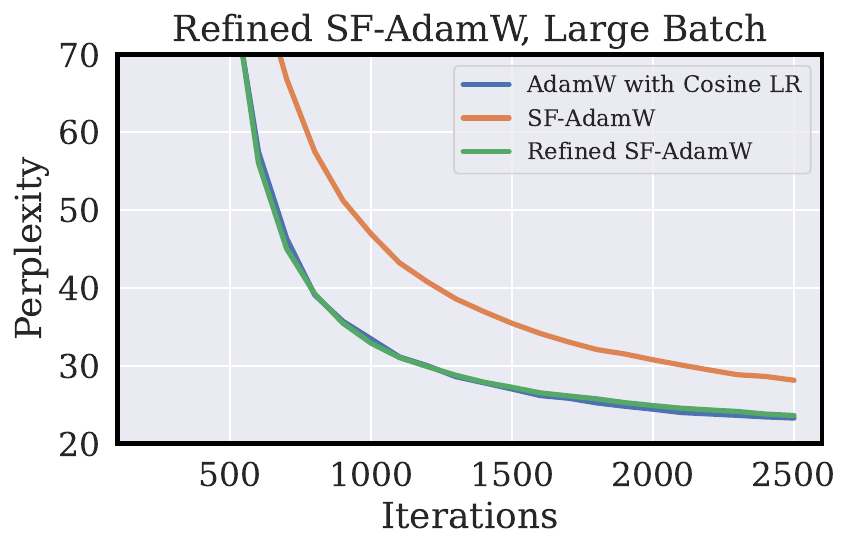}
    \end{subfigure}
    \caption{\textbf{OpenWebText2 Experiment: Refined SF-AdamW.} \textbf{Left:} Performance of $\bx_t$ and $\by_t$ iterates using the refined \texttt{SF-AdamW} with $\beta_1=0.5$ and $C = 50$. 
    \textbf{Middle:} Refined \texttt{SF-AdamW} with $(\beta_1,C) \in \{(0.95, 200), (0.9, 50)\}$ achieves improved performance over the best vanilla \texttt{SF-AdamW} run. \textbf{Right:} Under a large batch size (2M tokens), vanilla \texttt{SF-AdamW} with $\beta_1=0.98$ underperforms compared to \texttt{AdamW} with a cosine schedule, while the refined \texttt{SF-AdamW} with $C=500$ matches its final performance.}
    \label{fig:modified_openwebtext2}
\end{figure}

\clearpage
\subsection{Sensitivity to the Refinement Parameter \texorpdfstring{$C$}{C}}
\label{appendix:refined_sensitivity}

To examine the sensitivity of refined \texttt{SF-AdamW} to the choice of the refinement parameter $C$, we conduct experiments on SlimPajama (0.5M and 2M token batch sizes). We sweep a range of $C$ values and report test perplexity over 9M held-out tokens. The results are summarized in \autoref{tab:c_sweep_05M} and \autoref{tab:c_sweep_2M}.

Our results demonstrate that refined \texttt{SF-AdamW} is consistently robust to the choice of $C$. Across a broad range of settings, it outperforms vanilla \texttt{SF-AdamW} (i.e., $C=1/(1-\beta_1)$). For instance, in \autoref{tab:c_sweep_05M}, values such as $C \in \{50, 100, 200, 500\}$ improve performance for $\beta_1 = 0.9, 0.95$, and $C=50, 100$ improve performance across all momentum values. Similar trends hold in \autoref{tab:c_sweep_2M}, with performance gains persisting even at large $C$. These results indicate that refined \texttt{SF-AdamW} remains effective without requiring careful tuning of $C$.

\begin{table}[h]
    \centering
    \small
    \begin{tabular}{@{}c|cccccccc@{}}
        \toprule
        \textbf{$\beta_1$} & Vanilla & $C=5$ & 10 & 20 & 50 & 100 & 200 & 500 \\
        \midrule
        0.1 & 67.20 & 37.06 & 35.80 & 36.62 & 36.15 & 37.27 & -- & -- \\
        0.5 & 41.01 & -- & -- & 29.32 & 30.87 & 29.96 & 29.57 & -- \\
        0.9 & 27.70 & -- & 27.70 & -- & 23.97 & 24.64 & 24.93 & 25.11 \\
        0.95 & 25.12 & -- & -- & 25.12 & 23.98 & 23.60 & 24.37 & 24.83 \\
        \bottomrule
    \end{tabular}
    \vspace{5pt}
    \caption{\textbf{Refined \texttt{SF-AdamW} on SlimPajama with 0.5M-token batch size.} Test perplexity (on 9M held-out tokens) under varying $C$ values. LR = 2e-3, $\beta_2 = 0.99$.}
    \label{tab:c_sweep_05M}
\end{table}

\begin{table}[h]
    \centering
    \small
    \begin{tabular}{@{}c|ccccccccc@{}}
        \toprule
        \textbf{$\beta_1$} & Vanilla & $C=5$ & 10 & 20 & 50 & 100 & 200 & 500 & 1000 \\
        \midrule
        0.1 & 110.8 & 54.80 & 47.37 & 44.31 & 43.80 & -- & -- & -- & -- \\
        0.5 & 54.24 & 47.80 & 38.42 & 38.86 & 38.49 & 42.97 & -- & -- & -- \\
        0.9 & 31.34 & -- & 31.34 & 29.86 & 27.68 & 28.16 & 27.02 & 27.88 & -- \\
        0.95 & 27.29 & -- & -- & 27.29 & 25.77 & 25.45 & 25.77 & 27.31 & -- \\
        0.98 & 26.09 & -- & -- & 30.23 & 26.09 & 25.49 & 24.51 & 23.95 & 23.88 \\
        \bottomrule
    \end{tabular}
    \vspace{5pt}
    \caption{\textbf{Refined \texttt{SF-AdamW} on SlimPajama with 2M-token batch size.} Test perplexity (on 9M held-out tokens) under varying $C$ values. LR = 2e-3, $\beta_2 = 0.99$.}
    \label{tab:c_sweep_2M}
\end{table}

\clearpage
\section{Omitted Derivations and Proofs}\label{appendix:proof}

\subsection{Proof of Stability Threshold of Schedule-Free Optimizer}\label{appendix:eos}

\subsubsection{Proof of \cref{prop:stability_sfgd}}
\stabilityGD*

\begin{proof}
    To begin with, we revisit the update rule of \texttt{SF-GD}, given by
\begin{align} 
    \begin{split}
    \mathbf{x}_{t} &= (1-c_{t})\, \mathbf{x}_{t-1} + c_{t}\, \mathbf{z}_{t},\\[1mm]
    \mathbf{y}_t &= (1-\beta)\, \mathbf{z}_t + \beta\, \mathbf{x}_t,\\[1mm]
    \mathbf{z}_{t+1} &= \mathbf{z}_t - \gamma \nabla f(\mathbf{y}_t).
    \end{split}
    \tag{\texttt{SF-GD}}\label{eq:schefree_gd}
\end{align}
On a quadratic objective, we get $\nabla f(\mathbf{w})=\bH \bw + \bg$. By substituting this and combining the last two relations, we get
\begin{align*}
    \bz_{t+1} &= \bz_t - \gamma(\bH \by_t + \bg) \\
    &= \bz_t - \gamma((1-\beta)\bH\bz_t + \beta \bH \bx_t + \bg) \\
    &= (\bI - \gamma (1-\beta)\bH)\bz_t - \beta\gamma\bH\bx_t-\gamma\bg.
\end{align*}

By substituting this to the first relation, we get a recurrence relation governing $\mathbf{x}_t$ as follows:
\begin{align}
    \mathbf{x}_{t+1} =& (1-c_{t+1})\bx_t + c_{t+1}\bz_{t+1} \notag\\
    =&(1-c_{t+1})\bx_t + c_{t+1}((\bI - \gamma (1-\beta)\bH)\bz_t - \beta\gamma\bH\bx_t-\gamma\bg) \notag\\
    =&((1-c_{t+1})\bI-\beta\gamma c_{t+1} \bH)\bx_t + c_{t+1}(\bI - \gamma(1-\beta)\bH)\bz_t-\gamma c_{t+1}\bg \notag\\
    =& ((1-c_{t+1})\bI-\beta\gamma c_{t+1} \bH)\bx_t + c_{t+1}(\bI - \gamma(1-\beta)\bH)\left(\frac{1}{c_t}\bx_t - \left(\frac{1}{c_t}-1\right)\bx_{t-1}\right)-\gamma c_{t+1}\bg \notag\\
    =& \left(\left(1+\frac{c_{t+1}}{c_t}-c_{t+1}\right)\bI-\gamma c_{t+1}\left(\frac{1-\beta}{c_t}+\beta\right)\bH\right)\mathbf{x}_t \notag \\
    & +\left(\left(c_{t+1} - \frac{c_{t+1}}{c_t}\right)\bI-\gamma\left(c_{t+1} - \frac{c_{t+1}}{c_t}\right)(1-\beta)\bH\right)\mathbf{x}_{t-1}-\gamma c_{t+1}\mathbf{g}.
    \label{eq:recur_sfgd_quad}
\end{align}
Define $(\mathbf{q}, a) \coloneqq (\mathbf{q}, \lambda_1(\bH))$ to be the largest eigenvector/eigenvalue pair of $\bH$ and $\Tilde{x}_t = \mathbf{q}^\top\mathbf{x}_t + \frac{1}{a}\mathbf{q}^\top \mathbf{g}$. Then the sequence $\{\mathbf q^\top \bx_t\}$ diverges if and only if the sequence $\{\Tilde{x}_t\}$ diverges. By multiplying $\mathbf{q}^\top$ on both sides of \cref{eq:recur_sfgd_quad}, we get
\begin{align*}
    \mathbf{q}^\top\mathbf{x}_{t+1} =& \left(1+\frac{c_{t+1}}{c_t}-c_{t+1}-\gamma c_{t+1}\left(\frac{1-\beta}{c_t}+\beta\right)a\right)\mathbf{q}^\top\mathbf{x}_t \\
    & + \left(c_{t+1} - \frac{c_{t+1}}{c_t}\right)(1-\gamma(1-\beta)a)\mathbf{q}^\top\mathbf{x}_{t-1}-\gamma c_{t+1}\mathbf{q}^\top\mathbf{g},
\end{align*}
from $\mathbf{q}^\top \bH = a\mathbf{q}^\top$. From the definition of $\Tilde{x}_t$, we get
\begin{align*}
    \Tilde{x}_{t+1} = \left(1+\frac{c_{t+1}}{c_t}-c_{t+1}-\gamma c_{t+1}\left(\frac{1-\beta}{c_t}+\beta\right)a\right)\Tilde{x}_t + \left(c_{t+1} - \frac{c_{t+1}}{c_t}\right)(1-\gamma(1-\beta)a)\Tilde{x}_{t-1},
    \label{eq:diff_tilde_x}
\end{align*}
which is a linear time-varying second order difference equation governing $\Tilde{x}_t$.

Its asymptotic behavior is governed by the limiting recurrence relation:
\begin{align*}
    \bar{x}_{t+1} = (2-a(1-\beta)\gamma)\bar{x}_t + (a(1-\beta)\gamma-1)\bar{x}_{t-1}.
\end{align*}
Two roots of this recurrence relation are given by
\begin{align*}
     \lambda_1 = \frac{2-a(1-\beta)\gamma + \sqrt{(2-a(1-\beta)\gamma )^2 -4(a(1-\beta)\gamma -1)}}{2}\\
     \lambda_2 = \frac{2-a(1-\beta)\gamma - \sqrt{(2-a(1-\beta)\gamma )^2 -4(a(1-\beta)\gamma -1)}}{2}.
\end{align*}
If $2<a(1-\beta)\gamma <4+2\sqrt{2}$, then 
\begin{align*}
    |\lambda_1|^2 = |\lambda_2|^2 &= \left(\frac{2-a(1-\beta)\gamma}{2}\right)^2 - \frac{(2-a(1-\beta)\gamma )^2 -4(a(1-\beta)\gamma -1)}{4} \\
    &= a(1-\beta)\gamma-1>1,
\end{align*}
which implies $\bar{x}_t$ diverges.

If $a(1-\beta)\gamma \geq 4+2\sqrt{2}$, then $\lambda_2$ can be regarded as a real valued function with respect to $a(1-\beta)\gamma$. Since $\lambda_2$ is decreasing and $\lambda_2<-1$ when $a(1-\beta)\gamma=4+2\sqrt{2}$, we get $\lambda_2<-1$ when $a(1-\beta)\gamma \geq 4+2\sqrt{2}$, which implies that $\bar{x}_t$ also diverges.

Since we take $\gamma>0$ and $0\leq\beta<1$, the condition $a > \frac{2}{(1-\beta)\gamma}$ implies diverging $\bar{x}_t$ as well as $\Tilde{x}_t$.
\end{proof}

\subsubsection{Proof of \cref{prop:stability_sfpgd}}
To show \cref{prop:stability_sfpgd}, we first prove the following reparameterization lemma.

\begin{lemma}
    Define $\textnormal{\texttt{SF-PrecondGD}}$ by \eqref{eq:schefree} with $\Delta_t \triangleq \mathbf P^{-1} \nabla f(\by_t)$. Let $\{\bx_t\}$ denotes the iterates of \textnormal{\texttt{SF-PrecondGD}} on the objective $f(\bw)$, and let $\{\tilde{\bx}_t\}$ denote the iterates of \textnormal{\texttt{SF-GD}} on the reparameterized objective $\tilde{f}(\bw)=f(\mathbf P^{-1/2}\bw)$ with initialization $\tilde{\bx}_1=\mathbf P^{1/2}\bx_1$. Then, we have $\tilde{\bx}_t = \mathbf P^{1/2} \bx_t$ for all steps $t$.
    \label{lemma:reparam}
\end{lemma}

\begin{proof}
    We claim that the equivalence $(\Tilde{\bx}_t, \Tilde{\by}_t, \Tilde{\bz}_t)=(\mathbf{P}^{1/2}\bx_t, \mathbf{P}^{1/2}\by_t, \mathbf{P}^{1/2}\bz_t)$ holds for all $t$, where $(\Tilde{\bx}_t, \Tilde{\by}_t, \Tilde{\bz}_t)$ denotes the iterates on the reparametrized objective.
    
    For $t=1$, it holds from the definition. Assume that the equivalence holds at $t$. Then, the update for $\Tilde{\bz}_{t+1}$ is given by
    \begin{align*}
        \Tilde{\bz}_{t+1} &= \Tilde{\bz}_t - \gamma \nabla \Tilde{f}(\Tilde{\by}_t) \\
        &=\Tilde{\bz}_t - \gamma \bP^{-1/2} \nabla f(\bP^{-1/2}\Tilde{\by}_t) \\
        &=\bP^{1/2}\bz_t - \gamma \bP^{-1/2} \nabla f(\by_t) \quad \text{(inductive hypothesis)} \\
        &= \bP^{1/2}(\bz_t - \gamma \bP^{-1} \nabla f(\by_t)) \\
        &= \bP^{1/2} \bz_{t+1}.
    \end{align*}

    Meanwhile,
    \begin{align*}
        \Tilde{\bx}_{t+1} &= (1-c_{t+1})\Tilde{\bx}_t + c_t \Tilde{\bz}_{t+1} \\
        &= \bP^{1/2}((1-c_{t+1})\bx_t + c_t \bz_{t+1}) \\
        &= \bP^{1/2} \bx_{t+1}\\
        \Tilde{\by}_{t+1} &= (1-\beta)\Tilde{\bz}_{t+1} + \beta \Tilde{\bx}_{t+1} \\
        &= \bP^{1/2}((1-\beta)\bz_{t+1} + \beta \bx_{t+1}) \\
        &= \bP^{1/2}\by_{t+1},
    \end{align*}
    which proves the claim.
\end{proof}

\stabilityPreconGD*

\begin{proof}
    Recall that \texttt{SF-PrecondGDW} is defined by \eqref{eq:schefree} with $\Delta_t \triangleq \mathbf P^{-1} \nabla f(\by_t)+\lambda\by_t$, which is identical to
    \begin{align*}
        \mathbf P^{-1} \nabla f(\by_t)+\lambda\by_t = \mathbf{P}^{-1}\nabla g(\by_t), 
    \end{align*}
    where $g(\mathbf{w}) = f(\mathbf{w}) + \frac{1}{2}\lambda\|\mathbf{P}^{1/2}\mathbf{w}\|^2$. Therefore, \texttt{SF-PrecondGDW} is identical to \texttt{SF-PrecondGD} on the objective $g(\mathbf{w})$.
    
    Let $\{\tilde{\bx}_t\}$ be the iterates of \texttt{SF-GD} on the reparameterized objective $\tilde{g}(\mathbf{w}) = g(\mathbf{P}^{-1/2}\mathbf{w})$ with initialization $\tilde{\bx}_1 = \mathbf{P}^{1/2}\bx_1$. From \cref{lemma:reparam}, $\{\bx_t\}$, the iterates of \texttt{SF-PrecondGDW}, satisfy $\tilde{\bx}_t = \mathbf{P}^{1/2} \bx_t$, which implies that if $\tilde{\bx}_t$ diverges then $\bx_t$ also diverges. 

    From \cref{prop:stability_sfgd}, if $\lambda_1(\mathbf{P}^{-1}\mathbf{H} + \lambda\mathbf{I}) > \frac{2}{(1-\beta)\gamma}$, then $\tilde{\bx}_t$ diverges. This proves the claim.
\end{proof}

\subsection{Deriving the Reformulation of Schedule-Free Optimizer}\label{appendix:sf_reformulation}
We begin by defining the momentum variable:
\begin{align*}
\mathbf{m}_t \triangleq \frac{\mathbf{x}_{t} - \mathbf{z}_{t+1}}{\gamma}.
\end{align*}
Using the update rule $\mathbf{x}_t = (1 - c_t)\mathbf{x}_{t-1} + c_t \mathbf{z}_t$ and $\mathbf{z}_{t+1} = \mathbf{z}_t - \gamma \Delta_t$, we can express $\mathbf{m}_t$ recursively:
\begin{align*}
\mathbf{m}_t = \frac{\mathbf{x}_t - \mathbf{z}_{t+1}}{\gamma} = \frac{\mathbf{x}_t - \mathbf{z}_t}{\gamma} + \Delta_t
= (1 - c_t)\frac{\mathbf{x}_{t-1} - \mathbf{z}_t}{\gamma} + \Delta_t = (1 - c_t) \mathbf{m}_{t-1} + \Delta_t.
\end{align*}
Next, we derive an update rule for $\mathbf{y}_t$:
\begin{align*}
    \mathbf{y}_{t+1} &= (1-\beta)\, \mathbf{z}_{t+1} + \beta\, \mathbf{x}_{t+1}\\
    &= (1-\beta)\, (\mathbf{z}_{t} -\gamma \Delta_t) + \beta\, ((1-c_{t+1})\mathbf{x}_{t} +c_{t+1} \mathbf{z}_{t+1})\\
    &= (1-\beta)\mathbf{z}_{t} +\beta \mathbf{x}_{t}  -(1-\beta)\gamma \Delta_t + \beta c_{t+1} (\mathbf{z}_{t+1}-\mathbf{x}_{t})\\
    &= \mathbf{y}_t   -\gamma\left[ \beta c_{t+1} \mathbf{m}_t  +(1-\beta)\Delta_t\right].
\end{align*}
Hence, the update rule \eqref{eq:schefree} can be equivalently written as \eqref{eq:schefree_y}:
\begin{align*}
\mathbf{m}_t &=  (1-c_t) \mathbf{m}_{t-1} +\Delta_{t} , \\
    \mathbf{y}_{t+1}   &= \mathbf{y}_t   -\gamma  \beta c_{t+1} \mathbf{m}_t  -\gamma (1-\beta)\Delta_t . 
\end{align*}
Finally, we express $\mathbf{x}_t$ as a weighted average of $\mathbf{y}_t$. Starting from the original definitions:
\begin{align*}
\mathbf{x}_{t+1} &= (1 - c_{t+1}) \mathbf{x}_t + c_{t+1} \mathbf{z}_{t+1}, \\
\mathbf{z}_{t+1} &= \frac{\mathbf{y}_{t+1} - \beta \mathbf{x}_{t+1}}{1 - \beta},
\end{align*}
we substitute and obtain:
\begin{align*}
\mathbf{x}_{t+1}
= (1 - c_{t+1}) \mathbf{x}_t + c_{t+1} \left( \frac{\mathbf{y}_{t+1} - \beta \mathbf{x}_{t+1}}{1 - \beta} \right), 
\end{align*}
thus we conclude that
\begin{align*}
\mathbf{x}_{t+1}
&= \frac{(1 - c_{t+1})(1 - \beta) \mathbf{x}_t + c_{t+1} \mathbf{y}_{t+1}}{(1 - c_{t+1})(1 - \beta) + c_{t+1}}.
\end{align*}

\subsection{Deriving the Central Flow of Schedule-Free Optimizer}\label{appendix:central_flow}

\subsubsection{Deriving the Central Flow of Schedule-Free GD}

We begin with the reformulated update rule for \texttt{SF-GD}, as derived from \eqref{eq:schefree_y}:
\begin{align*}
\mathbf{m}_t &=  (1-c_t) \mathbf{m}_{t-1} + \nabla f(\by_t), \\
\mathbf{y}_{t+1} &= \mathbf{y}_t   -\gamma  \beta c_{t+1} \mathbf{m}_t  -\gamma (1-\beta)\nabla f(\by_t).
\end{align*}

As in gradient descent, stable training dynamics are often well-approximated by their continuous-time analogs. We can therefore define a corresponding \emph{stable flow} for \texttt{SF-GD}:
\begin{align*}
    \frac{d \by}{dt} &= -\gamma (1-\beta)\nabla f(\by) - \gamma \beta c(t+1) \mathbf m, \\
    \frac{d \mathbf m}{dt} &= \frac{1}{1-c(t)}\left[ \nabla f(\by)  - c(t) \mathbf m\right].
\end{align*}

However, at the Edge of Stability, the optimization trajectory deviates from this stable flow. We now derive a \emph{central flow} to characterize the time-averaged behavior of \texttt{SF-GD} under this regime, particularly when a single top eigenvalue remains at the stability threshold. This derivation is not rigorous but follows the ansatz approach used by \citet{cohen2025understanding}.

We model the trajectory as $\by_t = \bar{\by}_t + \rho_t \mathbf u_t$, where $\bar{\by}_t$ is the time-averaged iterate, $\mathbf u_t$ is the top Hessian eigenvector at $\bar{\by}_t$, and $\rho_t$ is the scalar displacement along $\mathbf u_t$. By construction, $\E [\rho_t] = 0$. Using a Taylor expansion of $\nabla f(\by)$ around the reference point $\bar{\by}$, we obtain:
\begin{align*}
    \nabla f(\by) = \nabla f(\bar{\by}) + \rho S(\bar{\by}) \mathbf u + \tfrac{1}{2} \rho^2 \nabla S(\bar{\by}) + \mathcal O (\rho^3),
\end{align*}
where $S(\by):=\lambda_1 (\nabla^2f(\by))$ denotes the sharpness at $\by$.
Taking expectations, the time-averaged gradient norm becomes:
\begin{align*}
    \E[\nabla f(\by_t)] &\approx \nabla f(\bar{\by}) + \E[\rho_t] S(\bar{\by}) \mathbf u + \tfrac{1}{2} \E[\rho_t^2] \nabla S(\bar{\by}) = \nabla f(\bar{\by})+ \tfrac{1}{2} \E[\rho_t^2] \nabla S(\bar{\by}). 
\end{align*}
Based on these approximations, we can derive the following central flow dynamics of $\bar{\by}_t$:
\begin{align*}
    \frac{d \by}{dt} &= -\gamma (1-\beta) \left[\nabla f(\by) + \tfrac{1}{2}\sigma^2(t)\nabla S(\by)\right] - \gamma \beta c(t+1) \mathbf m, \\
    \frac{d \mathbf m}{dt} &= \frac{1}{1-c(t)}\left[ \nabla f(\by) + \tfrac{1}{2}\sigma^2(t)\nabla S(\by) - c(t) \mathbf m\right],
\end{align*}
where $\sigma^2(t)$ models $\E[\rho_t^2]$, the instantaneous variance of the oscillations around the central flow trajectory (i.e., the river trajectory).

Recall that at the Edge of Stability, the sharpness equilibrates near the stability threshold. We therefore assume that it remains constant along the central flow trajectory, satisfying
\[
S(\by) = \frac{2}{(1-\beta)\gamma},\quad 
\frac{d}{dt}\left( S(\by) \right) = 0.
\]
There exists a unique value of $\sigma^2(t)$ that ensures this condition holds, particularly in the regime where $t$ is large, where the coefficient $c(t)=1/t$ becomes negligible. To compute this value of $\sigma^2$, we apply the chain rule and substitute the central flow dynamics:
\begin{align*}
\frac{dS(\by)}{dt}= \inner{\nabla S(\by), \frac{d\by}{dt}} \approx \inner{\nabla S(\by), -\gamma (1-\beta) \left[\nabla f(\by) + \tfrac{1}{2}\sigma^2(t)\nabla S(\by)\right]},
\end{align*}
where we approximate $c(t)\approx 0$. Setting $\frac{dS(\by)}{dt} = 0$ and rearranging gives:
\begin{align*}
    \sigma^2 (\by) \approx \frac{2\inner{\nabla S(\by), -\nabla f(\by)}}{\|\nabla S(\by)\|^2}.
\end{align*}

\subsubsection{Deriving the Central Flow of Schedule-Free Scalar Adam}

We begin with the reformulated update rule for \texttt{SF-ScalarAdam}, as derived from \eqref{eq:schefree_y}:
\begin{align*}
\nu_{t} &= \beta_2\nu_{t-1}+(1-\beta_2)\|\nabla f(\by_t)\|^2, \\
\mathbf{m}_t &=  (1-c_t) \mathbf{m}_{t-1} + \frac{\nabla f(\by_t)}{\sqrt{\nu_t}}, \\
\mathbf{y}_{t+1} &= \mathbf{y}_t   -\gamma  \beta_1 c_{t+1} \mathbf{m}_t  -\gamma (1-\beta_1)\frac{\nabla f(\by_t)}{\sqrt{\nu_t}}.
\end{align*}

As in gradient descent, stable training dynamics are often well-approximated by their continuous-time analogs. We can therefore define a corresponding \emph{stable flow} for \texttt{SF-ScalarAdam}:
\begin{align*}
    \frac{d \by}{dt} &= -\frac{\gamma (1-\beta_1)}{\sqrt{\nu}} \nabla f(\by) - \gamma \beta_1 c(t+1) \mathbf m, \\
    \frac{d \mathbf m}{dt} &= \frac{1}{1-c(t)}\left[ \frac{1}{\sqrt{\nu}} \nabla f(\by)  - c(t) \mathbf m\right], \\
    \frac{d\nu}{dt} &= \frac{1-\beta_2}{\beta_2} \left[\|\nabla f(\by)\|^2 - \nu\right],
\end{align*}

However, at the Edge of Stability, the optimization trajectory deviates from this stable flow. We now derive a \emph{central flow} to characterize the time-averaged behavior of \texttt{SF-ScalarAdam} under this regime, particularly when a single top eigenvalue remains at the stability threshold. This derivation is not rigorous but follows the ansatz approach used by \citet{cohen2025understanding}.

We model the trajectory as $\by_t = \bar{\by}_t + \rho_t \mathbf u_t$, where $\bar{\by}_t$ is the time-averaged iterate, $\mathbf u_t$ is the top Hessian eigenvector at $\bar{\by}_t$, and $\rho_t$ is the scalar displacement along $\mathbf u_t$. By construction, $\E [\rho_t] = 0$. Using a Taylor expansion of $\nabla f(\by)$ around the reference point $\bar{\by}$, we obtain:
\begin{align*}
    \nabla f(\by) = \nabla f(\bar{\by}) + \rho S(\bar{\by}) \mathbf u + \tfrac{1}{2} \rho^2 \nabla S(\bar{\by}) + \mathcal O (\rho^3),
\end{align*}
where $S(\by):=\lambda_1 (\nabla^2f(\by))$ denotes the sharpness at $\by$.
Taking expectations, the time-averaged gradient and squared gradient norm become:
\begin{align*}
    \E[\nabla f(\by_t)] &\approx \nabla f(\bar{\by}) + \E[\rho_t] S(\bar{\by}) \mathbf u + \tfrac{1}{2} \E[\rho_t^2] \nabla S(\bar{\by}) = \nabla f(\bar{\by})+ \tfrac{1}{2} \E[\rho_t^2] \nabla S(\bar{\by}), \\
    \E[\|\nabla f(\by_t)\|^2] &\approx \|\nabla f(\bar{\by})\|^2 + 2\E[\rho_t] S(\bar{\by}) \inner{\nabla f(\bar{\by}),\mathbf u} +  \E[\rho_t^2] S(\bar{\by})^2 = \|\nabla f(\bar{\by})\|^2+ \E[\rho_t^2]  S(\bar{\by})^2.
\end{align*}
Based on these approximations, we can derive the following central flow dynamics of $\bar{\by}_t$:
\begin{align*}
    \frac{d \by}{dt} &= -\frac{\gamma (1-\beta_1)}{\sqrt{\nu}} \left[\nabla f(\by) + \tfrac{1}{2}\sigma^2(t)\nabla S(\by)\right] - \gamma \beta_1 c(t+1) \mathbf m, \\
    \frac{d \mathbf m}{dt} &= \frac{1}{1-c(t)}\left[ \frac{1}{\sqrt{\nu}} (\nabla f(\by) + \tfrac{1}{2}\sigma^2(t)\nabla S(\by)) - c(t) \mathbf m\right], \\
    \frac{d\nu}{dt} &= \frac{1-\beta_2}{\beta_2} \left[\|\nabla f(\by)\|^2 + \sigma^2(t)S(\by)^2 - \nu\right],
\end{align*}
where $\sigma^2(t)$ models $\E[\rho_t^2]$, the instantaneous variance of the oscillations around the central flow trajectory (i.e., the river trajectory).

Recall that at the Edge of Stability, the preconditioned sharpness equilibrates near the stability threshold. We therefore assume that it remains constant along the central flow trajectory, satisfying
\[
\frac{S(\by)}{\sqrt{\nu}} = \frac{2}{(1-\beta_1)\gamma},\quad 
\frac{d}{dt}\left( \frac{S(\by)}{\sqrt{\nu}} \right) = 0.
\]
There exists a unique value of $\sigma^2(t)$ that ensures this condition holds, particularly in the regime where $t$ is large, where the coefficient $c(t)=1/t$ becomes negligible. To compute this value of $\sigma^2$, we apply the chain rule and substitute the central flow dynamics:
\begin{align*}
\frac{d}{dt}\left( \frac{S(\by)}{\sqrt{\nu}} \right) =& \frac{1}{\sqrt{\nu}} \inner{\nabla S(\by), \frac{d\by}{dt}} - \frac{S(\by)}{2\nu^{3/2}} \cdot \frac{d\nu}{dt}\\
\approx& \frac{1}{\sqrt{\nu}} \inner{\nabla S(\by), -\frac{\gamma (1-\beta_1)}{\sqrt{\nu}} \left[\nabla f(\by) + \tfrac{1}{2}\sigma^2\nabla S(\by)\right]} \\
&- \frac{S(\by)}{2\nu^{3/2}} \cdot \frac{1-\beta_2}{\beta_2} \left[\|\nabla f(\by)\|^2 + \sigma^2S(\by)^2 - \nu\right],
\end{align*}
where we approximate $c(t)\approx 0$. Setting $\frac{d}{dt}\left( \frac{S(\by)}{\sqrt{\nu}} \right) = 0$ and rearranging gives:
\begin{align*}
    \sigma^2 \approx \frac{\inner{\nabla S(\by), -\nabla f(\by)} + \frac{1-\beta_2}{2(1-\beta_1)\beta_2\gamma} \left[S(\by)\sqrt{\nu} - \frac{1}{\sqrt{\nu}}S(\by)\|\nabla f(\by)\|^2\right]}{\frac{1}{2}\|\nabla S(\by)\|^2 + \frac{1-\beta_2}{2(1-\beta_1)\beta_2\gamma}\cdot\frac{S(\by)^3}{\sqrt{\nu}}}.
\end{align*}
Using the condition $\frac{S(\by)}{\sqrt{\nu}} = \frac{2}{(1-\beta_1)\gamma}$, we substitute $\sqrt{\nu} = \tfrac{1}{2}(1-\beta_1)\gamma S(\by)$ into the expression and obtain:
\[
\sigma^2(\by;\beta_1,\beta_2,\gamma) \approx \frac{\inner{\nabla S(\by), -\nabla f(\by)} + \tfrac{1-\beta_2}{\beta_2}[\tfrac{1}{4}S(\by)^2 - \tfrac{1}{(1-\beta_1)^2\gamma^2}\|\nabla f(\by)\|^2 ]}{\tfrac{1}{2}\|\nabla S(\by)\|^2 + \tfrac{1-\beta_2}{(1-\beta_1)^2\beta_2\gamma^2}S(\by)^2}.
\]
Notably, $\sigma^2$ depends only on the current iterate $\by$ and the hyperparameters $\beta_1$, $\beta_2$, and $\gamma$. Moreover, unlike \texttt{SF-GD}, $\sigma^2$ \emph{does} depend on momentum parameters.

\subsection{Omitted Calculations in \cref{sec:refined}}\label{appendix:refined}
We derive the closed-form approximation of the averaging weights $\alpha_t$ under the modified SF method, where the update coefficient is set to
\[
c_t=\frac{(1-\beta)C}{t}.
\]
Under this setting, we show that the induced averaging weights satisfy the approximation
\[
\alpha_t \approx \frac{C}{T}\left(\frac{t}{T}\right)^{C-1}.
\]
Recall that for general $\{c_t\}$, the iterates $\bx_T$ can be written as a weighted average of past $\by_t$:
\[
\bx_T
  =\sum_{t=1}^{T}\!\alpha_t\,\by_t,\quad
  \alpha_t:=\frac{c_t}{(1-c_t)(1-\beta)+c_t}\!\prod_{s=t+1}^{T}\!\left[\frac{(1-c_s)(1-\beta)}{(1-c_s)(1-\beta)+c_s}\right].
\]
Now, substitute $c_t = \frac{(1 - \beta) C}{t}$. For large $t$, we approximate:
\begin{align*}
\alpha_t &= \frac{\frac{C}{t}}{1-\frac{(1-\beta)C}{t}+\frac{C}{t}}\left[\prod_{s=t+1}^T \frac{1-\frac{(1-\beta)C}{s}}{1+\frac{\beta C}{s}} \right] \\
&\approx \frac{C}{t}\left[\prod_{s=t+1}^T \frac{1-\frac{(1-\beta)C}{s}}{1+\frac{\beta C}{s}} \right] \\
&\approx \frac{C}{t} \left[\prod_{s=t+1}^T \frac{\exp \left(-\frac{(1-\beta)C}{s}\right)}{\exp \left(\frac{\beta C}{s}\right)} \right] \\
&= \frac{C}{t}\left[\prod_{s=t+1}^T \exp \left(-\frac{C}{s}\right) \right] \\
&= \frac{C}{t} \exp \left(-\sum_{s=t+1}^T \frac{C}{s} \right).
\end{align*}
Using the integral approximation for the harmonic sum:
\[
\sum_{s=t+1}^T\frac{1}{s} \approx \int_{s=t}^T \frac{1}{s},
\]
we obtain
\begin{align*}
\alpha_t &\approx \frac{C}{t} \exp \left(-\sum_{s=t+1}^T \frac{C}{s} \right) \\
&\approx \frac{C}{t} \exp \left( -\int_{s=t}^T \frac{C}{s} \right) \\
&= \frac{C}{t} \exp \left( -C\log T + C\log t \right) \\
&= \frac{Ct^{C-1}}{T^C} .
\end{align*}
Thus, we conclude that
\[
\alpha_t\;\approx\;\frac{C}{T}\Bigl(\frac{t}{T}\Bigr)^{C-1}.
\]
\cref{fig:modified_weights_histogram} and \cref{fig:modified_weights_histogram_T5000} show that the averaging weights $\{\alpha_t\}$ in our refined SF method closely follow the approximation $\alpha_t \approx \tfrac{C}{T}(\tfrac{t}{T})^{C-1}$, across different values of $\beta$, $C$, and $T$.

\begin{figure}[!tbhp]
    \centering
    \includegraphics[width=\linewidth]{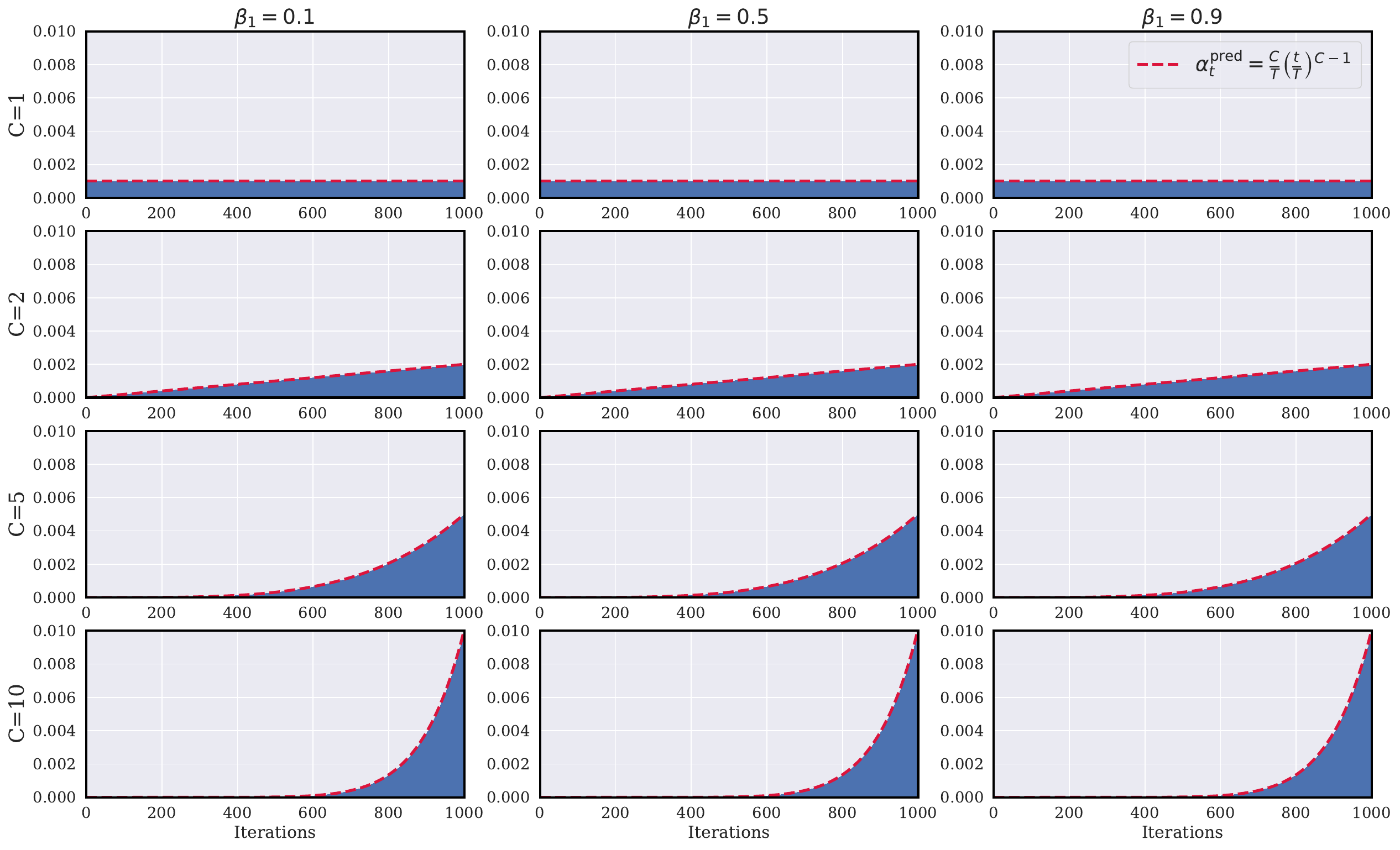}
    \vspace{-15pt}
    \caption{\textbf{Averaging weights in the refined SF method.} Histogram of $\{\alpha_t\}_{t=1}^T$ over $T = 1000$ iterations for varying values of $\beta$ and $C$.}
    \label{fig:modified_weights_histogram}
\end{figure}

\begin{figure}[!tbhp]
    \centering
    \includegraphics[width=\linewidth]{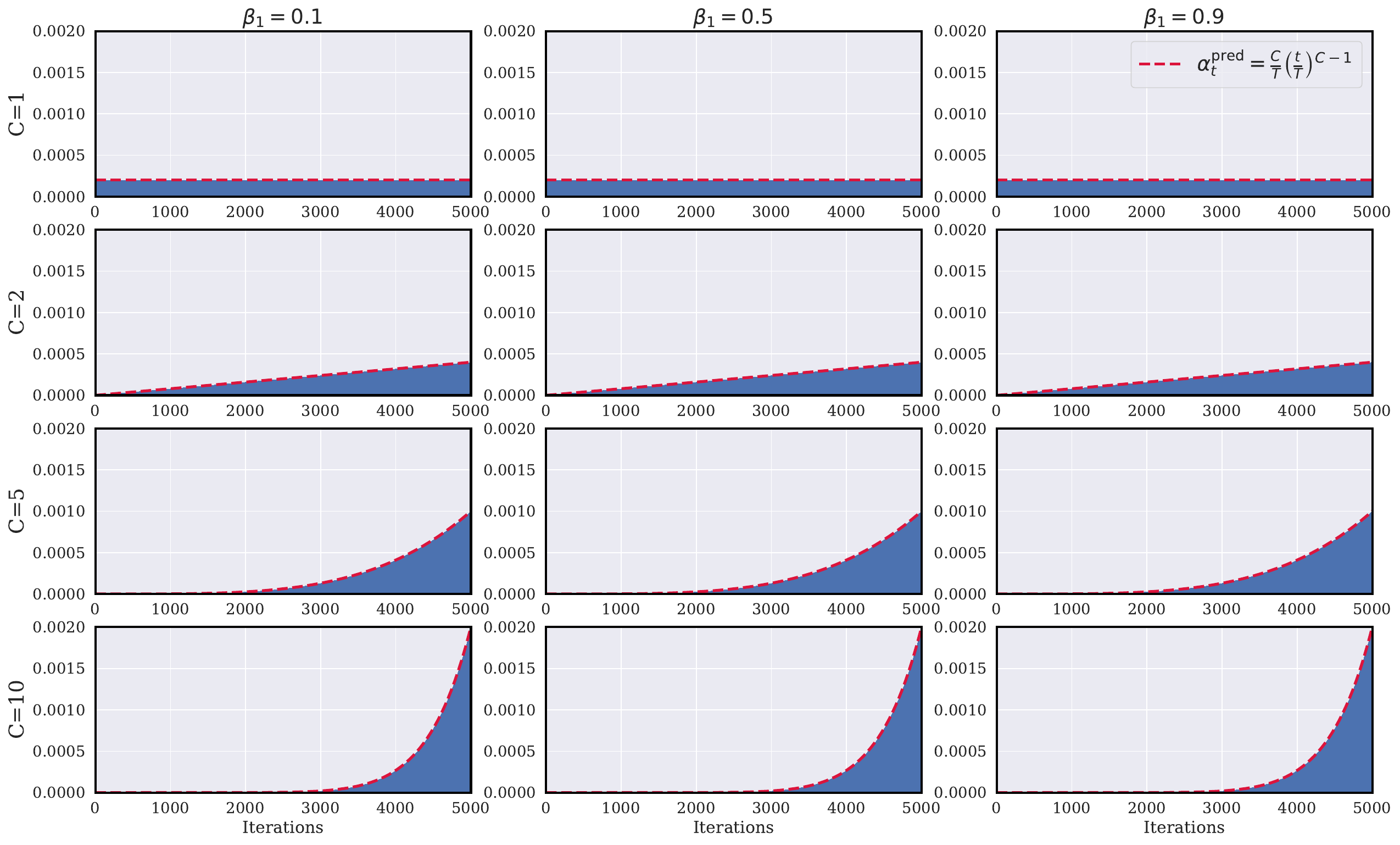}
    \vspace{-15pt}
    \caption{\textbf{Averaging weights in the refined SF method.} Histogram of $\{\alpha_t\}_{t=1}^T$ over $T = 5000$ iterations for varying values of $\beta$ and $C$.}
    \label{fig:modified_weights_histogram_T5000}
\end{figure}

\end{document}